\documentclass{article}
\usepackage[utf8]{inputenc}
\usepackage{amsmath, amsthm, amsfonts, xcolor, tikz, float, url, algorithm,algpseudocode, hyperref, accents, url,bbm, enumitem, amssymb, mathtools, caption, subcaption}
\usepackage{tikz}
\usepackage{tkz-graph}
\usetikzlibrary{shapes.geometric}
\usetikzlibrary{backgrounds}
\usetikzlibrary{arrows.meta}
\usepackage[round]{natbib}
\usepackage{fullpage}

\usepackage{xspace}

\newcommand\norm[1]{\left\lVert#1\right\rVert}

\newcommand{\nextver}[1]{{}}

\def\algon{\texttt{DecoR}\xspace}

\def\bfs{\texttt{BFS}\xspace}
\def\torrent{\texttt{Torrent}\xspace}
\def\tor{\texttt{Tor}\xspace}
\def\ols{\texttt{OLS}\xspace}
\def\outliers{G_n}
\def\outliersnn{G}

\def\inliers{G_n^{\mathsf{c}}}
\def\sampleset{\mathcal{U}_n}
\def\B{\mathrm{Inl}(\mathcal{U}_n)}

\newtheorem{theorem}{Theorem}
\newtheorem{remark}{Remark}

\newtheorem{assumption}{Assumption}
\newtheorem{proposition}{Proposition}
\newtheorem{definition}{Definition}
\newtheorem{setting}{Setting}

\newtheorem{corollary}{Corollary}
\newtheorem{lemma}{Lemma}

\DeclareMathOperator{\rank}{rank}

\DeclareMathOperator*{\argmin}{arg\,min}

\title{DecoR: Deconfounding Time Series with Robust Regression}

\author{Felix Schur\footnote{felix.schur@stat.math.ethz.ch}\\
        ETH Zurich, Switzerland \and Jonas Peters\\
        ETH Zurich, Switzerland}
\date{\today}

\begin{document}

\maketitle

\begin{abstract}
    Causal inference on time series data is a challenging problem, especially in the presence of unobserved confounders. This work focuses on estimating the causal effect between two time series that are confounded by a third, unobserved time series. Assuming spectral sparsity of the confounder, we show how in the frequency domain this problem can be framed as an adversarial outlier problem. We introduce \textbf{Deco}nfounding by \textbf{R}obust regression
    (\algon),
    a novel approach that estimates the causal effect using robust linear regression in the frequency domain.
    Considering two different robust regression techniques, we first improve existing bounds on the estimation error for such techniques. Crucially, our results do not require distributional assumptions on the covariates. We can therefore use them in time series settings.
    Applying these results to \algon, we prove,
    under suitable assumptions, upper bounds for the estimation error of \algon that imply consistency.
    We demonstrate \algon's effectiveness through experiments on 
    both synthetic and real-world data from Earth system science. The simulation experiments furthermore suggest that \algon is robust with respect to model misspecification. \\

    \noindent \textbf{KEYWORDS:}\\
    \noindent Causal inference, confounding, Earth system science, robust regression, time series analysis

\end{abstract}

\section{Introduction and Related Work}
Understanding causal relationships is a fundamental problem in many scientific 
disciplines ranging from economics and epidemiology to
biology and Earth system science.
Predicting a response from observations of covariates often falls short to answer the scientific question at hand. Instead, 
one often wants
to understand  how the response
reacts to interventions on the covariates, one of the core questions studied in causal inference
\citep{bb9b1ce1-c4cd-345c-b1f2-6bb227500876, pearl2009causality, peters2017elements}.
A recurring challenge in causal inference 
on non-randomized data
is the presence of unobserved confounders, that is, unobserved
variables
that influence both the predictor and the covariate and that
potentially lead
to biased estimates of the causal effect.

Instrumental variable (IV) regression offers a framework to 
remove bias due to hidden confounding 
by using instruments -- variables that influence the response variable only via the covariates of interest
\citep{wright1928tariff, reiersol1945confluence, bowden1990instrumental, angrist2009mostly}. 
Instrumental variables regression for time series data leads to
additional challenges due to temporal dependencies \citep{fair1970estimation, newey1986simple,thams2022identifying}.
In cases where instruments are not available, 
one may aim to exploit
alternative assumptions regarding the nature of the confounding.
For example, under the strong assumption of independent additive noise, \cite{janzig} propose a method that detects confounding in
i.i.d.~data.

In this paper, we assume
that the confounder is sparse in the frequency domain (or, more generally, after a suitable basis transformation). This assumption
allows us to frame the problem 
as an adversarial outlier problem %
in the frequency domain and thereby enabling us to apply 
robust regression techniques to estimate the causal effect reliably. 
Figure \ref{fig:main_plot} illustrates our proposed
method, \textbf{Deco}nfounding by \textbf{R}obust regression (\algon),
and offers a graphical comparison with ordinary least squares
(OLS). We analyse \algon theoretically and provide assumptions under which \algon is consistent for estimating the causal effect.

Our approach was inspired by \cite{mahecha2010global} who predict temperature sensitivity of ecosystem respiratory processes in the case where basal respiration, the unobserved confounder, is slowly varying. While \cite{mahecha2010global} also consider the time series data in the frequency domain, they %
assume that the support of the confounder is known (which is not required in our framework). Furthermore, they
focus on the application and
neither
provide a 
formal mathematical
framework nor any statistical guarantees.
The work by
\cite{cevid2020spectral} and \cite{scheidegger2023spectral} 
also consider sparsity with hidden confounders.
However, they consider i.i.d.~data and assume 
a sparse parameter vector
rather than sparse confounding. 
Consequently, their procedure is different from ours in that 
they apply singular value decomposition to the data and then trim the large eigenvalues.
A related line of research focuses on spatial deconfounding in environmental and epidemiological applications; see, e.g.,~\citet{clayton1993spatial, paciorek2010importance, page2017estimation}. 
These works assume
that a spatial regression problem is confounded by an unobserved confounder that is slowly varying. Different methods have been developed to solve this problem.
In \cite{reich2006effects}, \cite{hughes2013dimension} and \cite{prates2019alleviating} the residual spatial process is restricted to be orthogonal to the covariates. Another approach is to remove the slowly varying components from either the response or the covariates or both \citep{paciorek2010importance, thaden2018structural, keller2020selecting, dupont2022spatial+}. A similar 
methodology, that also includes removing the slowly varying part of the covariate, was
developed by \cite{sippel2019uncovering} for meteorological time series data.
More recently, \cite{guan2023spectral} and \cite{marques2022mitigating} consider Gaussian random fields as a data generation process and propose to remove confounding by considering different scales when estimating the (unconfounded) covariance matrix.

Our method reduces the unobserved confounder problem to
linear regression with adversarial outliers on non-i.i.d.~data. Linear adversarial outlier problems have
been studied extensively both in terms of methodology and theory.
In the setting where outliers are given by an oblivious adversary, that is, the outliers are not allowed to depend on the predictors, several consistent estimators are known \citep{bhatia2017consistent, suggala2019adaptive, d2021consistent}. However, when the outliers are chosen 
adversarially, that is, adaptively with respect to the predictors, there does not exist a consistent estimator when the 
fraction of data points contaminated by outliers and the noise variance are constant, even when the data is i.i.d.~(see Appendix~\ref{app:lower_bound}). 
Furthermore, in the i.i.d.~setting  with vanishing fraction of bounded outliers, 
standard \ols is consistent
(for details see Appendix~\ref{app:inconsistent}) -- even though it may be suboptimal in terms of finite sample results.
Surprisingly, for the non-i.i.d.~setting we are considering, we prove that there are cases for which robust regression is consistent while \ols is not, even when outliers are bounded. 
Even though it is generally impossible to construct consistent estimators, a variety of results have been derived for the linear adversarial outliers focusing mostly on the i.i.d.~setting.
\citet{klivans2018efficient} assume i.i.d.~data, contractivity constraints on the
distribution of the predictors and assume that the outcomes are bounded. The authors of \citet{chen2013robust} and \cite{diakonikolas2019efficient} assume (sub)-Gaussian design with uncorrelated predictors, and \citet{sasai2020robust} assume i.i.d.~Gaussian design and make restricted eigenvalue assumptions. Similarly, \citet{pensia2020robust} assume i.i.d.~data with contractivity constraints
and assume that the distance between contaminations are bounded. In \citet{bhatia2015robust} the authors assume eigenvalue bounds on the predictors and assume the contaminations to be bounded. 
One of the methods we study in more detail is the algorithm proposed by \citet{bhatia2015robust}: their results are among the few that do not require i.i.d.~data.
Not only do we improve the bounds given by \cite{bhatia2015robust} but we also prove that robust regression can be consistent with constant fraction of adversarial outliers and non-vanishing noise variance. The insight here is that when i.i.d.~noise with constant variance is considered in the frequency domain, the variance 
vanishes with increasing sample size.

The remainder of this paper is structured as follows.
We formally introduce the problem
of causal inference with unobserved confounders for time series in Section~\ref{sec:sparse_deconfounding}. In Section~\ref{sec:sparse_deconfounding_robust} we present \algon and showcase how it is used for causal inference in the presence of unobserved confounders. We provide theoretical guarantees for \bfs and \torrent in Section~\ref{sec:theoretical_g} and guarantees for \algon in Section~\ref{sec:theoretical_g_an}. In Section~\ref{sec:exp} we provide experimental results on both synthetic and real-world data.

\begin{figure}[t]
    \centering
    \includegraphics[width=12cm]{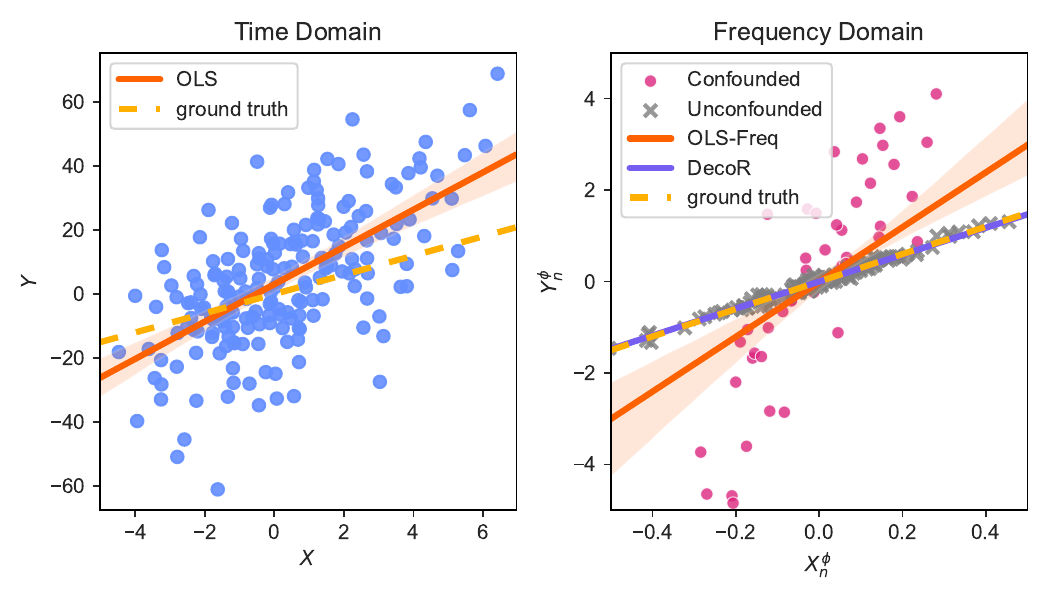}
        \caption{
        Due to the hidden confounder, regressing $Y$ on $X$ in the time-domain (left) does not yield a consistent estimator of the true causal effect (dashed, yellow). The idea of \algon (green, see Section~\ref{sec:sparse_deconfounding_robust}) is to consider the data in the frequency domain. Even though it is unknown, which of the data points correspond to confounded (orange) and unconfounded (gray) frequencies, robust regression techniques can be used to estimate the causal effect. We prove that \algon is consistent under weak assumptions if the confounding is sparse, see Section~\ref{sec:theoretical_g}.
        }
    \label{fig:main_plot}
\end{figure}

\section{Sparse Representation for Deconfounding}
\label{sec:sparse_deconfounding}
We now formalize the underlying 
model assumptions and the methodological procedure.
\begin{setting}
\label{set:robust_improved}
    Let $d \in \mathbb{N}$ and let $T \in \mathbb{R}_{>0}$ denote a fixed time horizon.
    Let $X = (X_t)_{t \in [0,T]}$ be a stochastic process in $\mathbb{R}^d$ and $U = (U_t)_{t \in [0,T]}$ a stochastic process in $\mathbb{R}$. Let $\eta = (\eta_t)_{t \in [0,T]}$ be a process of i.i.d.~centred Gaussian random variables independent of $X$ 
    and with constant variance $\sigma^2_{\eta} \geq 0$.
    Let $\beta \in \mathbb{R}^d$ and
    let $Y = (Y_t)_{t \in [0,T]}$ be a stochastic process that satisfies
    \begin{equation*}
        Y_t =  X_t^\top \beta  + U_t + \eta_t.
    \end{equation*}
    Fix $n \in \mathbb{N}$. We assume that we observe
    $X^n \coloneqq (X_{T/n}, X_{2T/{n}}, \dots, X_{T})$ and $Y^n:=(Y_{T/n}, Y_{2T/{n}}, \dots, Y_{T})$.
\end{setting}
We assume Setting~\ref{set:robust_improved} for the remainder of this section and consider the goal of estimating $\beta$ from $X^n$ and $Y^n$. Setting~\ref{set:robust_improved} 
does not require
any underlying causal model and 
$\beta$ is an interesting target of inference from a regression point of view. However, the set-up is particularly relevant if $\beta \in \mathbb{R}$ is the total causal effect of $X_t$ on $Y_t$, that is, $\beta = \frac{\partial}{\partial x}\mathbb{E}[Y_t \ | \ \mathrm{do}(X_t=x)]$ for all $t \in [0,T]$ \citep{pearl2009causality}.
In the most basic scenario, $U$ acts as a confounder, effecting $Y$ through a linear relationship while 
its effect on $X$ can be non-linear.
We illustrate this scenario in Figure~\ref{fig:DAG}.

However, Setting~\ref{set:robust_improved} 
is more general than the additive structure may initially suggest. In particular, we show that it suffices to assume that $Y_t = X_t^{\top} \beta + \epsilon_t$ for some (not necessarily uncorrelated) stochastic processes $X$ and $\epsilon$; for details, see Appendix~\ref{sec:more_general}.

Consider 
a (known) orthonormal basis
$\phi = \{\phi_k\}_{k \in \mathbb{N}}$ of $L^2([0, T])$.
We typically choose the 
cosine
basis 
(see Definition~\ref{def:cosine_basis})
as $\phi$ in our applications.
The 
main assumption of this paper is
that the confounder is sparse in this basis.
\begin{definition}[$(\phi, \outliersnn)$-sparse process %
]
\label{def:main}
    Let $\outliersnn \subseteq \mathbb{N}$ and let $U = (U_t)_{t \in [0,T]}$ be a stochastic process in $\mathbb{R}$ satisfying $\mathbb{E}[ \int_0^T U_t^2 dt] < \infty$. If for all $k \notin \outliersnn$ almost surely
    \begin{equation*}
        \langle U, \phi_k\rangle_{L^2} = 0,
    \end{equation*}
    we call $U$ a \emph{$(\phi, \outliersnn)$-sparse process}.
\end{definition}

\begin{assumption}[Sparse confounding assumption]
\label{ass:main}
    The set $G \subseteq \mathbb{N}$ is
    such that $U$ is a $(\phi, \outliersnn)$-sparse process.
\end{assumption}
The theoretical results presented in the remainder of this paper explicitly
assume
Assumption~\ref{ass:main} to hold for a $\outliersnn$ with suitable properties (informally speaking, for `small' $\outliersnn$).
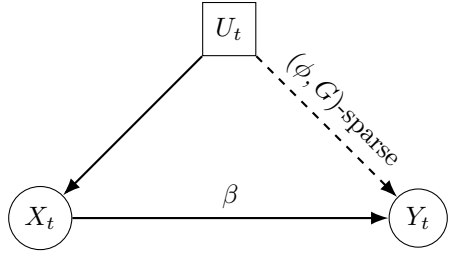
\begin{figure}[t]
  \centering
  \begin{tikzpicture}[
  ov/.style = {shape=circle,draw,minimum size=2em},
  uv/.style = {shape=rectangle,draw,minimum size=2em},
  edge/.style = {->, thick, -Latex},
  ]
    \node[ov] (X) at (0,0) {$X_t$};
    \node[ov] (Y) at (5,0) {$Y_t$};
    \node[uv] (U) at (2.5,2.5) {$U_t$};
    \draw[edge] (X) to node[above,sloped] {$\beta$} (Y);
    \draw[edge]  (U) -- (X);
    \draw[edge, dashed]  (U) to node[above,sloped] {$(\phi, \outliersnn)$-sparse} (Y);
    \end{tikzpicture}
\caption{
Directed acyclic graph
covered by
Setting~\ref{set:robust_improved}. The dashed arrow from $U_t$ to $Y_t$ indicates that we assume 
$U_t$ to be sparse and the effect on $Y_t$ to be additive,
}
\label{fig:DAG}
\end{figure}
To see how this sparsity can exploited, 
let %
$\mathcal{S}$ be the set of $\mathbb{R}^d$-valued stochastic processes\footnote{We work with general stochastic processes since we do not assume regularity conditions for $\eta$.}
on $[0, T]$ and $\mathcal{C}$ the set of random variables on $\mathbb{R}^d$. Define, for all $k \leq n$, the function $T_k^{\phi,n}: \mathcal{S} \to \mathcal{C}$ by
\begin{equation}
\label{eq:ft}
    T_k^{\phi,n }(V) \coloneqq 
    \begin{bmatrix}
        \frac{1}{n} \sum_{l=1}^n (V_{Tl/n})_1 \phi_k(Tl/n)\\
        \vdots\\
        \frac{1}{n} \sum_{l=1}^n (V_{Tl/n})_d \phi_k(Tl/n)
    \end{bmatrix}.
\end{equation}
For the cosine basis, Equation \eqref{eq:ft} corresponds to 
taking the discrete cosine transform (applied to $V_{T/n}, \ldots, V_{T}$). 
Using the linearity of $T_k^{\phi,n }$, it holds for all $k \leq n$
\begin{align}
\label{eq:core_idea}
\begin{split}
    T_k^{\phi,n }(Y) &= T_k^{\phi,n }(X)^{\top} \beta + T_k^{\phi,n }(U) + T_k^{\phi,n }(\eta)\\
    &\overset{\text{a.s.}}{=} 
    \begin{cases}
        T_k^{\phi,n }(X)^{\top}\beta + T_k^{\phi,n }(U) + T_k^{\phi,n }(\eta) & \text{if } k \in \outliersnn, \\
        T_k^{\phi,n }(X)^{\top}\beta + T_k^{\phi,n }(\eta) & \text{if } k \notin \outliersnn.
    \end{cases}
\end{split}
\end{align}
Equation~\eqref{eq:core_idea}
makes use of a technical condition on the transformation; see Assumption~\ref{ass:new}\ref{ass:new:2} 
and its discussion in Section~\ref{sec:theoretical_g_an}.

The core idea is to consider the pairs $\{(T_k^{\phi,n }(X), T_k^{\phi,n }(Y))\}_{k \leq n}$ as a new data set and to observe that by Assumption~\ref{ass:main} only the data points with an index $k \in G$ are confounded. If $\outliersnn$ is known or if there exists a known subset $S \subseteq \mathbb{N}$ with $G \subseteq S$, the unconfounded data set $\{(T_k^{\phi,n}(X), T_k^{\phi,n}(Y))\}_{k \notin S}$ can be analysed using standard linear inference methods to, for example,
consistently estimate $\beta$. However, in many cases, $\outliersnn$ (or $S$) might be unknown. In the next section, we introduce \algon, an algorithm for estimating %
$\beta$ without assuming knowledge of $\outliersnn$ but assuming sparsity instead.

\subsection{\algon: Deconfounding by Robust Regression %
}
\label{sec:sparse_deconfounding_robust}

The key insight exploited in
this section is the identification of \eqref{eq:core_idea} as an adversarial outlier problem.
We define $X^n_{\phi} \coloneqq (T_1^{\phi,n }(X), \dots, T_n^{\phi,n }(X))^{\top}$
and similarly for $Y^n_{\phi}$ and $\eta^n_{\phi}$, and
$o^n_{\phi}$ as $o^n_{\phi} \coloneqq (T_1^{\phi,n }(U), \dots, T_n^{\phi,n }(U))^{\top}$. This reformulation allows us to express the relationship \eqref{eq:core_idea} as:
\begin{equation}
\label{eq:hhgf}
    Y^n_{\phi} = X^n_{\phi}\beta + \eta^n_{\phi} + o^n_{\phi}.
\end{equation}
The outlier vector  $o^n_{\phi}$ may exhibit dependence on $X$. However, because of Assumption~\ref{ass:main} we know that for all indices $k$ not in $\outliersnn$, the component $(o^n_{\phi})_k$ equals $0$. This scenario mirrors a linear regression model with adversarial outliers, a challenge well-studied in robust statistics literature \citep[see, e.g.,][]{bhatia2015robust}.
To estimate $\beta$ we can thus apply robust linear regression 
to \eqref{eq:hhgf}. We call this methodology \algon 
and outline the full procedure in Algorithm~\ref{alg:algon}.
While it can be paired with any adversarial robust linear regression algorithm,
in our experiments, we have used \torrent \citep{bhatia2015robust}.
Figure~\ref{fig:algon_vis} in Appendix~\ref{app:figure}
offers a graphical illustration of the method, emphasizing the transition from time-domain data to frequency-domain analysis and the subsequent application of a robust regression algorithm.
\begin{algorithm}
\caption{\algon}\label{alg:algon}
\begin{algorithmic}
\Require $X^n \in \mathbb{R}^{n \times d}$, $Y^n\in \mathbb{R}^n$, orthonormal basis $\phi$, robust linear regression algorithm $\mathcal{A}$
\State $X^n_{\phi} \gets (T_1^{\phi,n }(X), \dots, T_n^{\phi,n }(X))^{\top}$ \Comment{See \eqref{eq:ft}}
\State $Y^n_{\phi} \gets (T_1^{\phi,n }(Y), \dots, T_n^{\phi,n }(Y))^{\top}$
\State $\hat{\beta}^{\phi,n}_{\algon} \gets \mathcal{A}(X^n_{\phi}, Y^n_{\phi})$
\State \Return $\hat{\beta}^{\phi,n}_{\algon}$
\end{algorithmic}
\end{algorithm}

In Section~\ref{sec:theoretical_g} we introduce the robust linear regression with adversarial outliers problem, discuss the robust algorithms \bfs and \torrent and provide novel theoretical guarantees for their estimation errors. In Section~\ref{sec:theoretical_g_an} we use these results to provide conditions under which the estimator returned by \algon is consistent.

\section{Theoretical Guarantees for Robust Regression}
\label{sec:theoretical_g}

We first introduce the setting of a linear model with adversarial outliers,
which we assume for the remainder of Section \ref{sec:theoretical_g}.
\begin{setting}
\label{set:outlier}
    Let $d \in \mathbb{N}$ and $\beta \in \mathbb{R}^d$. For all $n \geq d$, let $\epsilon \in \mathbb{R}^n$, $X \in \mathbb{R}^{n \times d}$ and $\outliers \subseteq \{1, \ldots, n\}$ and let $o \in \mathbb{R}^n$ be such that $\forall i \notin \outliers: o_i = 0$. Define 
    $Y \in \mathbb{R}^n$ by
    \begin{equation} \label{eq:linmodmain}
        Y \coloneqq X \beta+ \epsilon + o.
    \end{equation}
    We call $\outliers$ the \emph{(potential) outliers} and $\inliers$ the \emph{inliers}. We observe $X$ and $Y$. The goal is to estimate $\beta$.
\end{setting}
Setting~\ref{set:outlier} describes the adversarial outlier setting. In particular, $o$ can depend on $X$. For any $S \subseteq \{1, \ldots, n\}$, we denote by $X_S \in \mathbb{R}^{|S|\times d}$ the submatrix of $X$ that only includes rows with indices in $S$ 
(for $S = \emptyset$, this is the empty matrix, whose rank equals zero).
We similarly denote the subvectors of $Y, \epsilon$ and $o$ that only includes rows with indices in $S$ by $Y_S, \epsilon_S$ and $o_S$.
If $\rank(X_S) \geq d$, we denote by 
\begin{equation}
\label{eq:ols}
    \hat{\beta}^S_{\ols}(X, Y) \coloneqq (X_S^{\top} X_S)^{-1} X_S^{\top} Y_S
\end{equation}
the ordinary-least-squares estimator\footnote{If $X_{S}^{\top}X_{S}$ is not invertible, we use the Moore-Penrose inverse $(X_{S}^{\top}X_{S})^+$ instead; that is, more generally, we define $\hat{\beta}^S_{\ols}(X, Y) \coloneqq (X_S^{\top} X_S)^{+} X_S^{\top} Y_S$.}.
We often write $\hat{\beta}^S_{\ols}$ instead of $\hat{\beta}^S_{\ols}(X, Y)$ when $X,Y$ is apparent
from the context. We also write $\hat{\beta}_{\ols}(X,Y) \coloneqq \hat{\beta}^{\{1, \dots, n\}}_{\ols}(X,Y)$.

This work discusses two algorithms for estimating $\beta$ in~\eqref{eq:linmodmain},
both of these methods
include some form of linear regression. 
Algorithm~\ref{alg:bfrs}
takes as input a collection $\sampleset$ of candidate sets for the inliers $\inliers$.
For example, we might have knowledge of an upper bound $a$ for the number of outliers. We can then define $\sampleset = \{S \subseteq \{1,\dots, n\} \ \vert \ |S| = n -a\}$.
We can then search over all elements in $\sampleset$,
compute the \ols estimate on the corresponding data,
and choose the estimate with the smallest prediction error. We call this method Brute Force Search (\bfs) 
and detail it in Algorithm~\ref{alg:bfrs}.
We will see in Section~\ref{sec:worst_case} that if, indeed, $\inliers \in \sampleset$, the procedure is consistent under suitable assumptions.
\begin{algorithm}
\caption{\bfs}\label{alg:bfrs}
\begin{algorithmic}
\Require $X \in \mathbb{R}^{n}$, $Y\in \mathbb{R}^n, \sampleset$
\State $\hat{\beta}^n_{\bfs}(\sampleset) \gets 0$,
$\mathrm{err}_0 \gets \infty$
\For{$S \in \sampleset$}
    \State $\mathrm{err}_1 \gets \frac{1}{|S|}\norm{Y_S - X_S\hat{\beta}^{S}_{\ols}(X, Y)}^2_2$
\Comment{see \eqref{eq:ols}}
    \If{$\mathrm{err}_1 < \mathrm{err}_0$}
        \State $\hat{\beta}^n_{\bfs}(\sampleset) \gets \hat{\beta}_{S}$
        \State $\mathrm{err}_0 \gets \mathrm{err}_1$
    \EndIf
\EndFor
\State \Return $\hat{\beta}^n_{\bfs}(\sampleset)$
\end{algorithmic}
\end{algorithm}

Algorithm~\ref{alg:bfrs} is computationally expensive
unless $\sampleset$ contains only few sets.
In practice, \torrent \citep{bhatia2015robust} is an often computationally more efficient method for estimating $\beta$. 
For all $v \in \mathbb{R}^n$ let $s_1, \dots, s_n$ be the unique permutation of $\{1, \dots, n\}$ such that $v_{s_1} \leq v_{s_2} \leq \dots \leq v_{s_n}$, where ties are broken with a fixed deterministic rule.
Further, define for $a \in \{1,\dots,n \}$
the set
\begin{equation}
\label{eq:ht}
    \mathrm{HT}(v, a) \coloneqq \{s_1, \dots, s_a\}
\end{equation}
containing the indices corresponding to the $a$ smallest entries of $v$ (`HT' stands for `hard-threshold').
With this notation
\torrent
is defined as
in Algorithm~\ref{alg:torrent}.\footnote{We use the algorithm's original formulation here. One can, equivalently, instead of $e \gets |Y - X\hat{\beta}^{t}_{\tor}|$ take an element-wise square.}
\begin{algorithm}
\caption{\torrent \citep{bhatia2015robust} }
\label{alg:torrent}
\begin{algorithmic}
\Require $X \in \mathbb{R}^{n \times d}$, $Y\in \mathbb{R}^n$, $a \in \{1,\dots,n \}$
\State $S_0 \gets \{1,\dots,n \}$, $e \gets Y$, 
$\mathrm{err} \gets \infty$, $t \gets 0$
\While{$\norm{e_{S_t}}_2 < \mathrm{err}$}
    \State $t \gets t+1$
    \State $\mathrm{err} \gets \norm{e}_2$
    \State $\hat{\beta}^{t}_{\tor} \gets \hat{\beta}^{S_{t-1}}_{\ols}(X, Y)$ \Comment{see \eqref{eq:ols}}
    \State $v \gets |Y - X\hat{\beta}^{t}_{\tor}|$ 
    \State $S_{t} \gets \mathrm{HT}(v, a)$ \Comment{see \eqref{eq:ht}}
    \State $e \gets |Y_{S_t} - X_{S_t}\hat{\beta}^{t}_{\tor}|$
\EndWhile
\State \Return $\hat{\beta}^{n, a}_{\tor} := \hat{\beta}^t_{\tor}$
\end{algorithmic}
\end{algorithm}

\subsection{Guarantees for \bfs}
\label{sec:worst_case}
Let $\sampleset \subseteq \mathcal{P}(\{1, \dots, n\})$, $ \emptyset \notin \sampleset$,
be the collection of candidate sets for the inliers $\sampleset$ that we use as an input for Algorithm~\ref{alg:bfrs} and let us denote the algorithm's output by 
$\hat{\beta}^n_{\bfs}(\sampleset)$.
Furthermore, we define
\begin{equation*}
        \B \coloneqq \{S \in \sampleset \ \vert \ S \cap \outliers = \emptyset\},
\end{equation*}
denoting
the collection of sets in $\sampleset$ that do not contain any outliers.
We can now prove the following result about \bfs.

\begin{theorem}
\label{lem:ransac_gaussian}
    Assume Setting~\ref{set:outlier} with $d = 1$ %
    and that $\B \neq \emptyset$. Assume that $\epsilon_1, \ldots, \epsilon_n$ are i.i.d.~zero-mean Gaussians with variance $\sigma^2 \geq 0$.
    Define for all $S \in \sampleset$,  and $U_n \in \B$, %
    \begin{equation*}
        \alpha_1(S, U_n, \delta) \coloneqq \frac{|S|}{{|U_n|}} \left( \sigma \sqrt{|U_n|} + \sigma \sqrt{K\log(2|\sampleset|/\delta)}\right)^2 -\left( \sigma \sqrt{|S|} - \sigma \sqrt{K\log(2|\sampleset|/\delta)}\right)^2,
    \end{equation*}
    where $K>0$ is the constant from Lemma~\ref{thm:concentration},
    and define
    \begin{align*}
        \alpha(S, U_n, \delta)
        &\coloneqq
        \frac{4\sigma \norm{X_{{S}}}_2 \sqrt{2\log(2|\sampleset|/\delta)}}{\norm{X_{{S} \setminus \outliers}}_2^2}
        + \frac{\sqrt{\alpha_1(S, U_n, \delta)}}{\norm{X_{{S} \setminus \outliers}}_2}
        + \frac{\sigma \sqrt{2\log(2|\sampleset|/\delta)}}{\norm{X_{{S} \setminus \outliers}}_2}.
    \end{align*}
    Let $\delta >0$. Then, with probability at least $1-\delta$ it holds that\footnote{Here and below, we consider an upper bound to be infinite if it contains a summation term that is divided by zero.}
    \begin{equation*}
        \left| \hat{\beta}^n_{\bfs}(\sampleset) - \beta \right| \leq \max_{S \in \sampleset} \min_{U_n \in \B} \left\{ \alpha(S, U_n, \delta) + \frac{\sigma \sqrt{2\log(6|\sampleset|/\delta)}}{\norm{X_S}_2} \right\}.
    \end{equation*}
\end{theorem}
The proof can be found in Appendix~\ref{app:proofThm31}.
Theorem~\ref{lem:ransac_gaussian} implies that \bfs is consistent if
the noise variance $\sigma^2$ is equal to $0$ and $\min_{S \in \sampleset} ||X_{S \setminus \outliers}||_2 \neq 0$, 
implying that $\sampleset$ does not contain any subsets of $G_n$.
Furthermore, it shows that if
the fraction of outliers is vanishing with increasing $n$,
\bfs is consistent under mild assumptions on the 
collection $\sampleset$ of candidate sets;
it suffices,
for example,
if
the distribution of $X$ is such that\footnote{See Appendix~\ref{def:o_not} for a formal definition.},
\begin{equation}
\label{eq:bfs_cons_cond_2}
    \min_{S \in \sampleset} ||X_{S \setminus \outliers}||_2 \in \Omega_{\mathbb{P}}(\sqrt{n}), \qquad \max_{S \in \sampleset} ||X_{S}||_2 \in \mathcal{O}_{\mathbb{P}}(\sqrt{n})
\end{equation}
and if
$\sampleset$ is chosen such that
\begin{equation}
\label{eq:bfs_cons_cond}
    \max_{S \in \sampleset} \min_{U_n \in \B} \sqrt{|S|\log(|\sampleset|)/\sqrt{|U_n|}} \in o(\sqrt{n}).
\end{equation}
For details, see Appendix~\ref{app:3453456}. Conditions~\eqref{eq:bfs_cons_cond} and~\eqref{eq:bfs_cons_cond_2} hold, for example, when we have a known sequence $\{c_n\}_{n \in \mathbb{N}}$ such that, for all $n$,
$|\outliers| \leq c_n$ and $c_n$ $ \in o( \sqrt{n}/\log(n))$, we define the collection of candidate sets as $\sampleset \coloneqq \{ S \in \mathcal{P}(\{1,\dots,n\}) \mid |S| = n - c_n\}$ and the rows of $X$ are i.i.d.\ Gaussian, see \citet[Theorem 15]{bhatia2015robust}.
If the fraction of outliers is non-vanishing (and $c_n < n/2$), Theorem~\ref{lem:ransac_gaussian} implies that for a distribution of $X$ satisfying
\eqref{eq:bfs_cons_cond_2} and the same choice of candidate sets defined by $\sampleset \coloneqq \{ S \in \mathcal{P}(\{1,\dots,n\}) \mid |S| = n-c_n\}$, that $| \hat{\beta}^n_{\bfs}(\sampleset) - \beta| \in \mathcal{O}(\sigma)$.
Theorem~\ref{lem:ransac_gaussian} implies uniform consistency results, too. For example, in the case of $\sigma = 0$, we obtain uniform consistency even over classes of distributions with 
outliers that shift the original data points by an arbitrary amount.
On the contrary, regression that is based on a robust loss function, 
such as the Huber loss \citep{Huber64},
does not come with similar
guarantees, not even with respect to pointwise consistency or in the case $\sigma = 0$.

\subsection{Guarantees for \torrent}
\label{sec:scaling_n}
In this section we strengthen existing bounds for the estimation error of \torrent by improving the result of \citet[Theorem~10]{bhatia2015robust}.

\begin{lemma}
\label{lem:tor}
    For all $n \in \mathbb{N}$ there is a constant $N \in \mathbb{N}$ such that Algorithm~\ref{alg:torrent} converges in less than $N$ steps.
\end{lemma}
In practice, we observe that \torrent requires only very few iterations until it converges (cf. Table~\ref{tab:convergence}).
We can use Lemma~\ref{lem:tor}
to obtain an improved guarantee for Algorithm~\ref{alg:torrent}.
\begin{theorem}
\label{thm:tor}
    Assume Setting~\ref{set:outlier}.
    Assume that there exists a known $c_n \in \mathbb{N}$ such that $|\outliers| \leq c_n$.
    Let $S_t$ be the estimated subset in the final iteration of Algorithm~\ref{alg:torrent} executed on the data $X$, $Y$ 
    using threshold parameter $a_n \coloneqq n - c_n$. Define for $S \subseteq \{1, \dots, n\}$
    \begin{equation*}
        V(S) \coloneqq (S \cup \inliers) \setminus (\inliers \cap S),
    \end{equation*}
    the symmetric difference between $S$ and $\inliers$.
    Furthermore, assume that \footnote{For $A \in \mathbb{R}^{d \times d}$ we denote by $\lambda_{\mathrm{min}}(A)$ the minimum eigenvalue of $A$.}
    \begin{equation} \label{eq:8888888888888}
        \eta \coloneqq \max_{S \subseteq \{1,\dots,n\} \text{ s.t. } |S| = a_n}\frac{\norm{X_{V(S)}}_2}{\sqrt{\lambda_{\mathrm{min}}\left( X_{S}^T X_{S}\right)}} < \frac{1}{\sqrt{2}}.
    \end{equation}
    Then
    \begin{align*}
        \norm{\hat{\beta}^{n, a_n}_{\tor} - \beta} _2 \leq \norm{ \left(X_{S_{t}}^{\top} X_{S_{t}} \right)^{-1} X_{S_{t}}^{\top} \epsilon_{S_{t}} }_2  + \frac{\sqrt{2} \norm{X_{V(S_t)}}_2  \norm{ \left(X_{S_{t}}^{\top} X_{S_{t}} \right)^{-1} X_{S_{t}}^{\top} \epsilon_{S_{t}} }_2 + \sqrt{2}\|\epsilon_{V(S_t)}\|_2}{\sqrt{\lambda_{\mathrm{min}}\left( X_{S_t}^T X_{S_t}\right)}\left(1 - \sqrt{2}\eta \right)}.
    \end{align*}
\end{theorem}
The proof can be found in Appendix~\ref{app:proofThm33}.
Defining
\begin{equation*}
    \lambda_{a_n}(X) \coloneqq \min_{S \subseteq \{1,\dots,n\} \text{ s.t. } |S| = a_n}\sqrt{\lambda_{\mathrm{min}}\left( X_{S}^T X_{S}\right)}
\end{equation*}
we can use this result to derive an upper bound for the estimation error in the sub-Gaussian noise setting.

\begin{corollary}[sub-Gaussian noise]
\label{cor:torrent_gaussian}
    Assume the setting of Theorem~\ref{thm:tor} 
    (including the assumption that~\eqref{eq:8888888888888} holds),
    with i.i.d.~zero-mean sub-Gaussian noise $\epsilon$ with variance proxy $\sigma^2$ and assume $c_n < n/2$.
    Then, there exists a constant $K>0$ such that for all $\delta >0$ with probability at least $1-\delta$
    \begin{align*}
        \norm{\hat{\beta}^{n, a_n}_{\tor} - \beta} _2 &\leq \frac{\sigma}{\lambda_{a_n}(X)} \left( 1 + \frac{\sqrt{2}}{1 - \sqrt{2}\eta} \right)  \left(\sqrt{d} + \sqrt{2c_nK\log(2en/c_n\delta)}\right)
        +  \frac{ 2\sigma \sqrt{c_n} \left(1 + \sqrt{K\log(2en/c_n\delta)} \right)}{\lambda_{a_n}(X)(1 - \sqrt{2}\eta)}.
    \end{align*}
    In particular, if the rows of $X$ are i.i.d.~standard Gaussian random vectors and $c_n \in o(n/\log(n))$, then $\hat{\beta}^{n, a_n}_{\tor}$ 
    is consistent.
\end{corollary}
While Corollary~\ref{cor:torrent_gaussian} states consistency under appropriate conditions and vanishing fractions of outliers,
the results in \cite{bhatia2015robust}
are not sufficiently strong to imply consistency in this setting.
Furthermore, even for non-vanishing fraction of outliers our bounds are in general tighter since the maximum over $S$ in the definition of $\eta$ is taken jointly over the denominator and numerator, not separately as in \cite{bhatia2015robust}.

\section{Guarantees for \algon}
\label{sec:theoretical_g_an}
We can now prove bounds for the estimation error of \algon. Let $(\Omega, \mathcal{F}, \mathbb{P})$ be a probability space and let $T \in \mathbb{R}_{>0}$ denote a fixed time horizon.
Denote by $\mathcal{M}^2_{[0, T]}$ the set of measurable, real valued, square integrable stochastic processes, 
that is, all Lebesgue-measurable, real-valued stochastic processes $(V_t)_{t \in [0,T]}$ that satisfy
$\mathbb{E}[ \int_0^T V_t^2 dt] < \infty$.
For all
$\phi = \{\phi_k\}_{k \in \mathbb{N}}$ 
orthonormal bases of $L^2([0, T])$ and all $V = (V_t)_{t \in [0,T]} \in \mathcal{M}^2_{[0, T]}$, we have 
$\int_0^T V_t^2 dt < \infty$ almost surely
and, therefore, $\langle V, \phi_k\rangle_{L_2}$ exists almost surely.

For the remainder of this section we assume
Setting~\ref{set:robust_improved} and that
Assumption~\ref{ass:main} 
is
satisfied for $\outliersnn$ and $\phi$. 
More precisely, let $\outliersnn \subseteq \mathbb{N}$, let $\phi = \{\phi_k\}_{k \in \mathbb{N}}$ be an orthonormal basis of $L^2([0, T])$, let $U = (U_t)_{t \in [0,T]} \in \mathcal{M}^2_{[0, T]}$ be a $(\phi,\outliersnn)$-sparse process and let $X = (X_t)_{t \in [0,T]}$ be a stochastic process in $\mathbb{R}^d$.
Let $\eta = (\eta_j)_{j \in [0,T]}$ be
a stochastic process of i.i.d.~centred Gaussian random variables with variance $\sigma_{\eta}^2 \geq 0$ and let
$Y_t$ 
be a stochastic process, such that, 
for all $t \in [0,T]$,
\begin{equation*}
    Y_t \coloneqq X_t^{\top} \beta + U_t + \eta_t.
\end{equation*}
We also assume some regularity conditions on $\phi$ and $U$. Specifically, we assume that for all $k \in \mathbb{N}$ it holds that $\phi_k$ is right-continuous (or left-continuous)
and that the 
trajectories of $U$ satisfy the same condition almost surely. 
This implies that almost surely\footnote{
Without assuming right-continuity \eqref{eq:hghgh} only holds almost surely, almost everywhere; because we need to discretize the process, we require equality almost surely for all $t \in  [0,T]$.
}, for all $t \in [0,T]$
\begin{equation}
\label{eq:hghgh}
    U_t = \sum_{k \in G} \langle \phi_k, U\rangle_{L^2}\phi_k(t).
\end{equation}
As described in Setting~\ref{set:robust_improved}, we assume that, for $n \in \mathbb{N}$, we observe
$X^n \coloneqq (X_{T/n}, X_{2T/{n}}, \dots, X_{T})$ and $Y^n:=(Y_{T/n}, Y_{2T/{n}}, \dots, Y_{T})$. 
We consider the transformation 
$T_k^{\phi,n }$, see~\eqref{eq:ft},
and write
$X^n_{\phi} \coloneqq (T_1^{\phi,n }(X), \dots, T_n^{\phi,n }(X))^{\top}$ (see also Algorithm~\ref{alg:algon}) and
$\outliers \coloneqq G \cap \{1,\dots,n\}$.

The theoretical results presented in this section are based on two types of assumptions. One assumption ensures that $\outliers$ does not grow too quickly with $n$ (clearly, if $|\outliers| \geq n/2$, for example, there is no consistent estimator for $\beta$).
Another
set of assumptions contains technical conditions 
on the transformation $T^{\phi,n}$, see~\eqref{eq:ft} and~\eqref{eq:core_idea}. 
\begin{assumption}
\label{ass:new}
    \begin{enumerate}[label=(\roman*)]
        \item\label{ass:new:2} For all $n \in \mathbb{N}$ and $l,k \leq n$ we have $\frac{1}{n }\sum_{j=1}^n  \phi_l(Tj/n)\phi_k(Tj/n) = \mathbbm{1}{\{ l = k \}}$.
        \item\label{ass:new:4}
        For every $\delta >0$ there exist $c'>0$ and $\bar{n} \in \mathbb{N}$ such that for all $n \geq \bar{n}$ there exist $S'_n \subseteq \{1,\dots,n\}$ with $|S'_n| = 2c_n + d$ such that for all $S'' \subset S'_{n}$ with $|S''| = d$ it holds with probability at least $1-\delta$
        \begin{equation*}
            \lambda_{\mathrm{min}}\left((X^n_{\phi})_{S''}^{\top} (X^n_{\phi})_{S''} \right) \geq c'.
        \end{equation*}
    \end{enumerate}
\end{assumption}
Assumption~\ref{ass:new}~\ref{ass:new:2} states that
the chosen orthonormal basis 
maintains its orthogonality and normalization properties when applied to discretized observations; this is satisfied, for example, by the cosine basis or the Haar basis \citep{haar1909theorie}, see Appendix~\ref{sec:defs}. Together with~\eqref{eq:hghgh}, this implies \eqref{eq:core_idea}.
Assumption~\ref{ass:new}~\ref{ass:new:4} requires that $X$ is non-sparse in the $\phi$-domain. This condition intuitively says that there is more information in $X$ than is lost by confounding. This condition is relatively mild and accommodates a wide array of commonly used stochastic processes, including Ornstein-Uhlenbeck processes, Brownian motions, and band-limited processes.
We are now able to state the main theoretical results, the convergence properties
of \algon-\bfs (\algon with \bfs as the robust regression algorithm) and \algon-\tor (\algon with \tor as the robust regression algorithm).

\begin{theorem}[Convergence properties of \algon-\bfs]
\label{thm:end_improved_bfs}
    Let $c_n$ be a known sequence of natural numbers such that 
    $|\outliers| \leq c_n$ and 
    let Assumption~\ref{ass:new} be satisfied. Assume that $d = 1$, that $U$ is independent of $\epsilon$, that $\sup_{t \in [0,T]} \mathbb{E}[X_t^2] < \infty$ and that  for all $n \in \mathbb{N}$ and $j,i\leq T$ it holds that
    $\frac{1}{n}\sum_{l=1}^n \phi_l(Tj/n) \phi_l(Ti/n) = \mathbbm{1}{\{ i = j \}}$. Define a sequence of candidate sets by $\sampleset := \{ S \subseteq \{1, \dots, n\} \ \vert \ |S| = n-c_n \}$.
    If \algon is executed with \bfs and the sequence of sample sets $\sampleset$, then 
    \begin{equation*}
        \left| \hat{\beta}^{\phi,n}_{\algon} - \beta \right| \in \mathcal{O}_{\mathbb{P}}\left( \sigma_{\eta}\left(\frac{ c_n\log(n/c_n)}{n}\right)^{1/4} \right).
    \end{equation*}
\end{theorem}
Many commonly used bases, such as the cosine basis and the Haar basis, satisfy the conditions of Theorem~\ref{thm:end_improved_bfs}.
\begin{theorem}[Convergence properties of \algon-\tor]
\label{thm:end_improved_torrent}
    Let $c_n$ be a known sequence of natural numbers such that 
    $|\outliers| \leq c_n$ and let Assumption~\ref{ass:new} be satisfied.
    Assume that for all $\delta > 0$ there exists $\bar{n} \in \mathbb{N}$ such that for all $n \geq \bar{n}$ it holds that
        \begin{equation}
        \label{eq:ass_top}
            \mathbb{P}\left[\max_{S \subseteq \{1,\dots,n\} \text{ s.t. } |S| = n-c_n} \frac{\norm{(X^n_{\phi})_{V(S)}}_2}{\sqrt{\lambda_{\mathrm{min}}\left( (X^n_{\phi})_{S}^T (X^n_{\phi})_{S}\right)}} < \frac{1}{\sqrt{2}} \right] \geq 1 - \delta.
        \end{equation}
        If \algon is executed with \torrent and the sequence of threshold parameters $a_n = n - c_n$, then 
    \begin{align*}
        \norm{\hat{\beta}^{\phi,n}_{\algon} - \beta}_2 &
        \in \mathcal{O}_{\mathbb{P}}\left( \sigma_{\eta}\sqrt{\frac{ c_n\log(n/c_n)}{n}} \right).
    \end{align*}
\end{theorem}
If the assumptions of Theorem~\ref{thm:end_improved_bfs} or Theorem~\ref{thm:end_improved_torrent} are satisfied and either there is no noise, that is, $\sigma_{\eta} = 0$, or the number of confounded components $c_n$ grows more slowly than the sample size $n$, that is, $c_n \in o(n)$, then \algon produces a consistent estimator for $\beta$. Conversely, if there is non-zero noise and the number of confounded components $c_n$ is 
asymptotically
proportional to $n$, that is,
$c_n \sim n$, the estimation error of \algon is asymptotically bounded by the noise variance $\sigma_{\eta}$, up to constant factors.
While \algon-\tor has polynomial-time complexity and in our experiments converges very fast (cf. Table~\ref{tab:convergence}),
the theoretical result (Theorem~\ref{thm:end_improved_torrent})
contains
stronger theoretical assumptions than the one based on \bfs (Theorem~\ref{thm:end_improved_bfs})
, specifically it requires \eqref{eq:ass_top} to hold.
In contrast, \algon-\bfs does not require \eqref{eq:ass_top}, but its computational complexity is exponential in $n$ when $c_n$ is not bounded by a constant.

In the next section, we show that when $X$ and $U$ are band-limited processes, \algon can be consistent, even for a constant fraction of outliers.
\subsection{Example: Band-Limited Processes}
\label{sec:band-limited}
Let $\{\phi_k\}_{k \in \mathbb{N}}$ be a basis of $L^2([0,T])$ and let $S \subset  \mathbb{N}$ be a bounded set. We call a real-valued stochastic process $V$ an 
\emph{$(S, \phi)$ band-limited process}
if there exists a set of
random variables $\{a^V_k\}_{k \in S}$ with finite second moments such that, almost surely,
    \begin{equation*}
        V = \sum_{k \in S} a^V_k \phi_k.
    \end{equation*}
Let $S_X, S_U \subset \mathbb{N}$ be bounded sets and assume that $X$ is a $(S_X, \phi)$ band-limited process and that $U$ is a $(S_U, \phi)$ band-limited process. Assume $\phi$ satisfies Assumption~\ref{ass:new}~\ref{ass:new:2}. Then, for all $k,n \in \mathbb{N}$, we have almost surely that
\begin{align*}
    (X^{\phi}_n)_k =
    \begin{cases}
        a^X_k & \text{if $k \in S_X$}\\
        0 & \text{if $k \notin S_X$}
    \end{cases}
\end{align*}
and analogously for $U$.
Let $c \coloneqq |S_X \cap S_U|$ 
and assume that $|S_X| \geq 2c + 1$. Assume further that, 
there exists $r > 0$ such that for all $k \in S_X$ it holds that almost surely 
$|a_k^U| \geq r$. This implies Assumption~\ref{ass:new}~\ref{ass:new:4} and therefore, by Theorem~\ref{thm:end_improved_bfs},
\algon-\bfs is consistent. 
If, additionally, condition \eqref{eq:ass_top} is satisfied -- this is the case with high probability, for example, if $\{a_k^X\}_{k \in S_X}$ are i.i.d.~Gaussian 
and $|S_X|$ is large 
-- then, by Theorem~\ref{thm:end_improved_torrent}, \algon-\tor is also consistent.
Moreover, \ols can be inconsistent in this setting, as shown in the following proposition.
\begin{proposition}
\label{prop:ols_not_consistent}
    Assume Setting~\ref{set:robust_improved} and assume Assumption~\ref{ass:new} holds. Then
    $\hat{\beta}^{\phi,n}_{\ols}$
    is not necessarily a consistent estimator for $\beta$. 
    Even if $X$ is a band-limited process, \ols may be inconsistent.
\end{proposition}
We regard these results surprising in that they describe a setting in which the number of outliers increases linearly in $n$ and there is a constant\footnote{The noise variance is only constant in the time domain. In the frequency domain, it vanishes with increasing $n$.} noise variance, and still 
robust regression yields a consistent estimator (and \ols does not).
As discussed in Appendix~\ref{app:lower_bound} such an estimator does not exist for i.i.d.~data.

\section{Empirical Evaluation}
\label{sec:exp}
We validate our theoretical findings through experiments on synthetic data and real-world data. The data sets and code for these experiments have been made available at \url{https://github.com/fschur/robust_deconfounding}.

\subsection{Synthetic Data}
In our synthetic experiments, we generate data according to the model described in Setting~\ref{set:robust_improved} and set $X \coloneqq U + \epsilon_X$. We sample $\epsilon_X$ and $U$ from either two independent Ornstein-Uhlenbeck processes or two independent band-limited processes. For the orthogonal basis $\phi$, we consider both the cosine basis (see Definition~\ref{def:cosine_basis}) and the Haar basis (see Definition~\ref{def:haar_basis}). 
The noise in $Y$ is set to have variance
$\sigma_{\eta}^2 \in \{0,1,4\}$.
To assess the accuracy of the different approaches, we calculate the mean absolute prediction error: $ \mathrm{MAE} \coloneqq \sum_{i=1}^n |\beta_i - \hat{\beta}^n_i|/n$.
For the remainder of this section, if not specified otherwise, we assume independent band-limited processes, the cosine basis, noise variance of $\sigma^2_{\eta}=1$, fraction of outliers of $25\%$ and $a = 0.7$.
Detailed information on the experimental setup is available in Appendix~\ref{app:detail_experiments}.
\begin{table}
     \centering
        \begin{tabular}{ |c|c|c|c|c| } 
        \hline
        $n$ & $\sigma_{\eta}^2$ & \ols & \algon-\tor & \algon-\bfs \\
        \hline
        8& 0 &1.69 (0.05) & 0.32 (0.04)& 0.00 (0.00)\\ 
        12& 0 &1.70 (0.05) & 0.13 (0.02)& 0.00 (0.00)\\ 
        16& 0 &1.66 (0.04) & 0.06 (0.02) & 0.00 (0.00)\\
        \hline
        8& 1 &1.70 (0.05) & 0.55 (0.03)& 0.24 (0.01)\\ 
        12& 1 &1.71 (0.05) & 0.33 (0.02)& 0.17 (0.01)\\ 
        16& 1 &1.67 (0.04) & 0.21 (0.01) & 0.14 (0.00)\\
        \hline
        \end{tabular}
    \caption{ Comparison of the mean absolute estimation error (with standard deviation of the mean in parenthesis) between \ols, 
    \algon-\bfs and \algon-\tor. For more than $16$ data points \bfs becomes computationally infeasible.}
    \label{tab:bfs}
\end{table}

We first compare 
the average estimation error of \ols, \algon-\tor and \algon-\bfs.
As discussed in Section~\ref{sec:theoretical_g}, the computation time of \algon-\bfs scales exponentially with the sample size $n$,
so we only consider sample sizes up to $n = 16$. 
Table~\ref{tab:bfs} shows the results.
In the absence of noise, that is, if $\sigma^2_{\eta} = 0$, \algon-\bfs correctly point-identifies the true parameter $\beta$ for all samples sizes -- this is in agreement with Theorem~\ref{thm:end_improved_bfs}.
Although the average estimation error of \algon-\tor decreases as the sample size increases, it does not reach zero, unlike the estimation error of \algon-\bfs. This discrepancy arises because condition~\eqref{eq:ass_top} is often not satisfied in smaller sample sizes (not shown). 
The discrepancy also suggests a trade-off between computational costs and average estimation error. 
Since \algon is modular in that it can be applied with any robust regression technique, it directly benefits from any methodological advancement in the field of  robust regression.
As expected, the estimation error for \ols remains constant regardless of sample size and is significantly larger than that of both \algon-\tor and \algon-\bfs.

In Section~\ref{sec:band-limited} we have shown
that under suitable conditions \algon is consistent when $U$ and $X$ are band-limited processes.
Figure~\ref{fig:exp_n_False_cosine} (left)
shows that the mean absolute estimation error for processes from this class vanishes with increasing sample size, supporting
this result.
Figure~\ref{fig:exp_n_False_cosine} (right) 
considers 
Ornstein-Uhlenbeck processes and 
suggests that \algon might also be consistent for this class of
processes. 
The results of Appendix~\ref{app:lower_bound}, however,
suggest that there are settings with Ornstein-Uhlenbeck processes where no consistent estimator exists;
it therefore seems that many of the cases in 
Figure~\ref{fig:exp_n_False_cosine} (right) contain confounding that is not worst-case, so \algon is still consistent.
In both cases (Figure~\ref{fig:exp_n_False_cosine} (left) and Figure~\ref{fig:exp_n_False_cosine} (right)) the error of \ols does not seem to vanish with growing sample size.

We present additional experimental results in Appendix~\ref{app:additional_exp}. Figure~\ref{fig:exp_n_False_haar} and Figure~\ref{fig:exp_n_True_2} suggest that \algon-\tor is consistent in the setting where the Haar basis is used in place of the cosine basis, and in the setting where $X$ is two-dimensional. 
Furthermore, Figure~\ref{fig:exp_n_True_cosine_ablation} 
investigates consistency of \algon-\tor under model misspecification;
more specifically, we consider a growing fraction of confounded components and a model introducing Gaussian noise to $U$ (and therefore violating
Assumption~\ref{ass:main}).
We also test the convergence speed of \algon-\tor (see Figure~\ref{tab:convergence}). We find that \algon-\tor converges in under 15 iterations for data sets up to size 1000.
\begin{figure}
     \centering
     \begin{subfigure}[t]{0.45\textwidth}
         \centering
         \includegraphics[width=\textwidth]{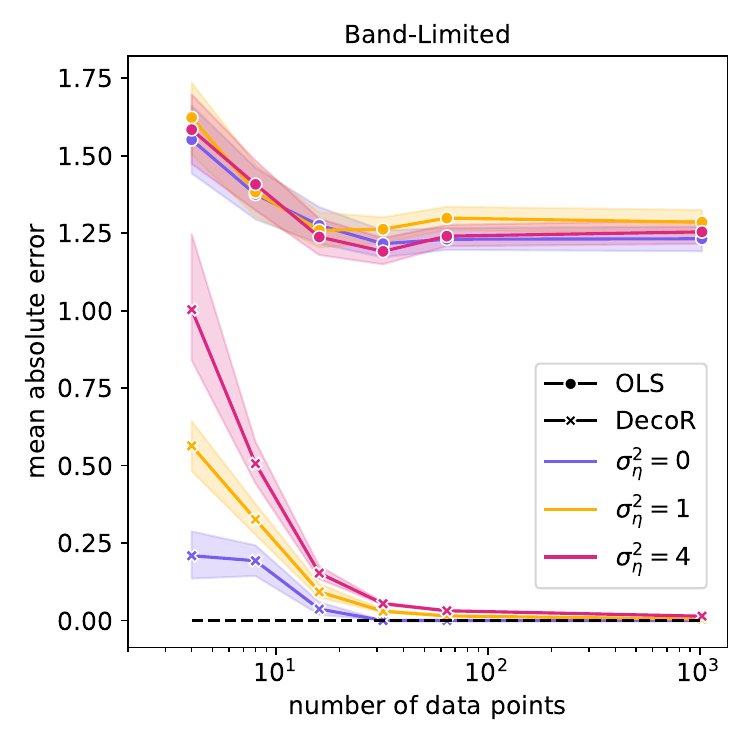}
     \end{subfigure}
     \hfill
     \begin{subfigure}[t]{0.45\textwidth}         
        \centering
         \includegraphics[width=\textwidth]{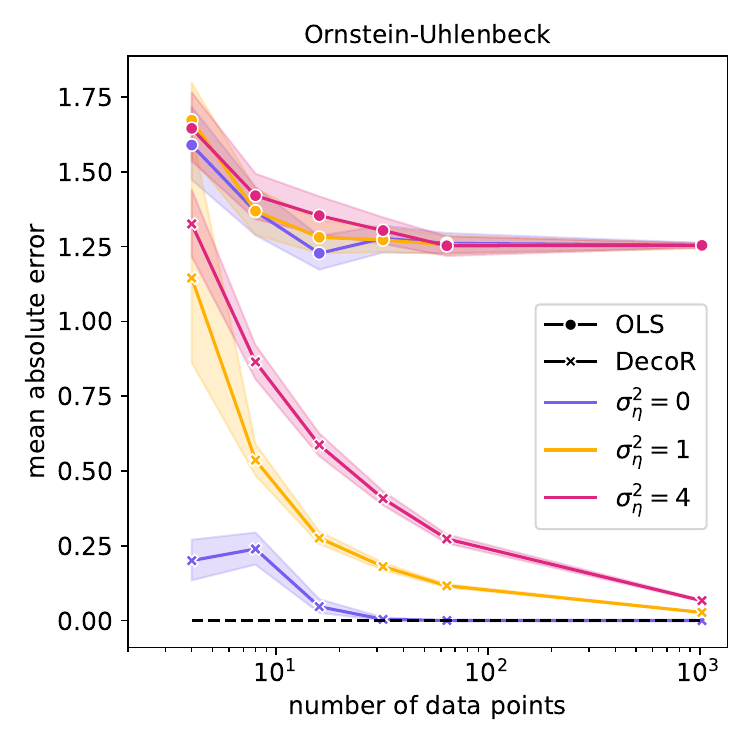}
     \end{subfigure}
     \caption{Synthetic experiment where $\epsilon_x$ and $U$ are generated by two independent band-limited processes (left) or two independent Ornstein-Uhlenbeck processes (right) and where we choose $\phi$ to be the cosine basis. For this experiment \algon-\tor is used.}
     \label{fig:exp_n_False_cosine}
\end{figure}

\subsection{Application to Earth System Science}

An important challenge in Earth System Science is to distinguish externally forced changes in climate systems from internal variability or ``climate noise'' \citep{madden1976estimates, schneider1994examination}. Drawing inspiration from \cite{sippel2019uncovering} we consider a specific instance involving atmospheric circulation and precipitation and frame the problem as a hidden confounder problem. Specifically, in our study, the covariate time series $X$ represents atmospheric circulation, in this case pressure at sea level, the target time series $Y$ corresponds to precipitation, and the hidden confounder $U$ denotes external forcing (such as greenhouse gas concentrations), which we presume to be sparse in the frequency domain. The goal is to estimate the internal circulation-induced contribution to variability of precipitation. Thus, after estimating the causal effect of $X$ on $Y$ we remove the confounded frequencies from $X$ and 
use the causal function from $X$ to $Y$ to produce fitted values; these fitted values correspond to the internal circulation-induced contribution to the variability of precipitation (assuming that the confounder acts linearly on $X$); the corresponding residuals of $Y$ can be used as an estimate of the total forced response.
We use the data set provided by the ClimEx project \citep{leduc2019climex}, based on a single run of a climate model, and consider precipitation data during
the winter months (December, January, and February) in the Alps and pressure at sea level data during the same time over Europe.
Our method, \algon-\tor with parameter $a = 0.9n$ and a cosine basis, is compared to the approach developed by \cite{sippel2019uncovering}. A critical aspect of the approach of \cite{sippel2019uncovering} is the prerequisite that the frequency bands in which the confounder is active are known; more specifically, one may assume that the confounder (external forcing) is smooth (see also \cite{wallace1995dynamic}), that is, the external forcing consists solely of low-frequency components.
If this assumption approximately holds, we expect
\algon to exclude
mainly low-frequency components when applied to the data set.
\begin{figure}
     \centering
     \begin{subfigure}[t]{0.45\textwidth}
         \centering
         \includegraphics[width=\textwidth]{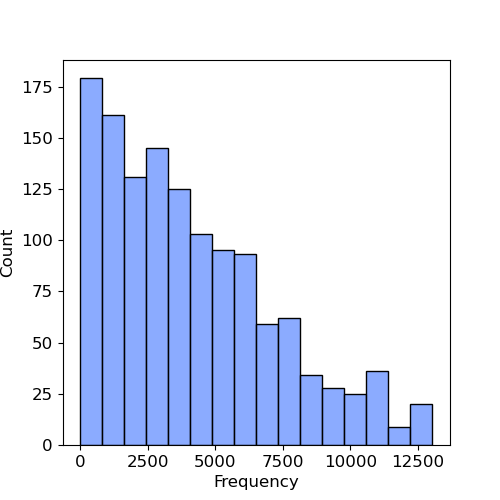}
     \end{subfigure}
     \hfill
     \begin{subfigure}[t]{0.45\textwidth}         
        \centering
         \includegraphics[width=\textwidth]{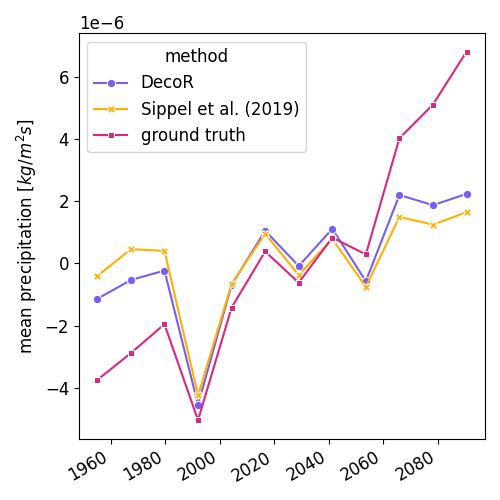}
     \end{subfigure}
     \caption{Histogram of the
     frequencies excluded by \algon in the final iteration (left) and centred average 5-year precipitation in the alps over time (right); 
     here, the 
     methods' outputs estimate the precipitation without external forcing. The ``ground truth'' data is the precipitation provided by the ClimEx project \citep{leduc2019climex}.
    }
     \label{fig:real_world}
\end{figure}
Indeed, an
analysis of the frequencies excluded by \algon, shown in Figure~\ref{fig:real_world}(a), confirms that our algorithm predominantly filters out the confounded low frequencies rather than the unconfounded high frequencies.
Furthermore, the coefficient of determination ($R^2$) for the fitted values (see above)
\algon equals $0.76$, the same as the one obtained by the method of \cite{sippel2019uncovering}.
Thus, \algon seems
competitive 
even without using the expert knowledge of which frequencies to exclude. 
The fitted values of the precipitation patterns (which ignore the confounded frequencies) are depicted in Figure~\ref{fig:real_world}(b); both \algon and the method used by \cite{sippel2019uncovering} estimate lower mean precipitation than observed in the data set, 
(which is to be expected, as it is believed that external forcing induces an upward trend in the data \citep{foster2011global, deser2016forced}).
In conclusion, our method matches the effectiveness of the approach by \cite{sippel2019uncovering} while imposing fewer assumptions.

\section{Summary and Future Work}
In this work we have developed \algon, an algorithm that estimates causal effects in the presence of
unobserved confounders in time series data. We leverage sparsity in
$L^2$ with respect to a known basis to derive conditions under which \algon consistently estimates the true causal effect.

Looking ahead, we see several promising directions for future work. First, 
adapting
our results to data sets where the data points are not observed in regular intervals. Similarly, extending our framework to include the asymptotics of
longer observational time horizons, as opposed to shorter time intervals, would be interesting. 
Second,
exploring scenarios where the effect of the hidden confounder has a sparse effect on $X$, but is dense toward $Y$, may be relevant for practical applications. 
We
hypothesize that, in this case, our approach
would be consistent under minor modifications.
Third,
it could be interesting to extend the results about the lower bound to other noise distributions such as the Gaussian distribution.
Fourth,
another promising direction is to consider nonlinear causal effects. In this setting, one strategy could be to generalize the derived results to finite-dimensional feature transformations or to kernelize the approach.

\section*{Acknowledgements}
We thank Nicolai Meinshausen
for his helpful input regarding the ideas of sparse confounding in the frequency domain and
the real data application and Michael Law for his helpful comments on convergence of the \ols estimator for time series data. We thank Maybritt Schillinger for helping us with the Earth System Science experiment. We also thank Pio Blieske for taking the time to read this paper and offering invaluable feedback.
We have used ChatGPT to improve minor language formulations.

\bibliographystyle{abbrvnat}
\bibliography{bibliography}

\begin{thebibliography}{45}
\providecommand{\natexlab}[1]{#1}
\providecommand{\url}[1]{\texttt{#1}}
\expandafter\ifx\csname urlstyle\endcsname\relax
  \providecommand{\doi}[1]{doi: #1}\else
  \providecommand{\doi}{doi: \begingroup \urlstyle{rm}\Url}\fi

\bibitem[Ahmed et~al.(1974)Ahmed, Natarajan, and Rao]{ahmed1974discrete}
N.~Ahmed, T.~Natarajan, and K.~R. Rao.
\newblock Discrete cosine transform.
\newblock \emph{IEEE transactions on Computers}, 100:\penalty0 90--93, 1974.

\bibitem[Angrist and Pischke(2009)]{angrist2009mostly}
J.~D. Angrist and J.-S. Pischke.
\newblock \emph{Mostly harmless econometrics: An empiricist's companion}.
\newblock Princeton University Press, 2009.

\bibitem[Bhatia et~al.(2015)Bhatia, Jain, and Kar]{bhatia2015robust}
K.~Bhatia, P.~Jain, and P.~Kar.
\newblock Robust regression via hard thresholding.
\newblock \emph{Advances in Neural Information Processing Systems}, 28, 2015.

\bibitem[Bhatia et~al.(2017)Bhatia, Jain, Kamalaruban, and
  Kar]{bhatia2017consistent}
K.~Bhatia, P.~Jain, P.~Kamalaruban, and P.~Kar.
\newblock Consistent robust regression.
\newblock \emph{Advances in Neural Information Processing Systems}, 30, 2017.

\bibitem[Bowden and Turkington(1990)]{bowden1990instrumental}
R.~Bowden and D.~Turkington.
\newblock \emph{Instrumental Variables}.
\newblock Econometric Society Monographs. Cambridge University Press, 1990.

\bibitem[{\'C}evid et~al.(2020){\'C}evid, B{\"u}hlmann, and
  Meinshausen]{cevid2020spectral}
D.~{\'C}evid, P.~B{\"u}hlmann, and N.~Meinshausen.
\newblock Spectral deconfounding via perturbed sparse linear models.
\newblock \emph{The Journal of Machine Learning Research}, 21:\penalty0
  9442--9482, 2020.

\bibitem[Chen et~al.(2013)Chen, Caramanis, and Mannor]{chen2013robust}
Y.~Chen, C.~Caramanis, and S.~Mannor.
\newblock Robust sparse regression under adversarial corruption.
\newblock In \emph{International Conference on Machine Learning}, pages
  774--782, 2013.

\bibitem[Clayton et~al.(1993)Clayton, Bernardinelli, and
  Montomoli]{clayton1993spatial}
D.~G. Clayton, L.~Bernardinelli, and C.~Montomoli.
\newblock Spatial correlation in ecological analysis.
\newblock \emph{International journal of epidemiology}, 22:\penalty0
  1193--1202, 1993.

\bibitem[Deser et~al.(2016)Deser, Terray, and Phillips]{deser2016forced}
C.~Deser, L.~Terray, and A.~S. Phillips.
\newblock Forced and internal components of winter air temperature trends over
  north america during the past 50 years: Mechanisms and implications.
\newblock \emph{Journal of Climate}, 29:\penalty0 2237--2258, 2016.

\bibitem[Diakonikolas et~al.(2019)Diakonikolas, Kong, and
  Stewart]{diakonikolas2019efficient}
I.~Diakonikolas, W.~Kong, and A.~Stewart.
\newblock Efficient algorithms and lower bounds for robust linear regression.
\newblock In \emph{Proceedings of the Thirtieth Annual ACM-SIAM Symposium on
  Discrete Algorithms}, pages 2745--2754, 2019.

\bibitem[Dupont et~al.(2022)Dupont, Wood, and Augustin]{dupont2022spatial+}
E.~Dupont, S.~N. Wood, and N.~H. Augustin.
\newblock Spatial+: a novel approach to spatial confounding.
\newblock \emph{Biometrics}, 78:\penalty0 1279--1290, 2022.

\bibitem[d’Orsi et~al.(2021)d’Orsi, Novikov, and Steurer]{d2021consistent}
T.~d’Orsi, G.~Novikov, and D.~Steurer.
\newblock Consistent regression when oblivious outliers overwhelm.
\newblock In \emph{International Conference on Machine Learning}, pages
  2297--2306, 2021.

\bibitem[Fair(1970)]{fair1970estimation}
R.~C. Fair.
\newblock The estimation of simultaneous equation models with lagged endogenous
  variables and first order serially correlated errors.
\newblock \emph{Econometrica: Journal of the Econometric Society}, pages
  507--516, 1970.

\bibitem[Foster and Rahmstorf(2011)]{foster2011global}
G.~Foster and S.~Rahmstorf.
\newblock Global temperature evolution 1979--2010.
\newblock \emph{Environmental research letters}, 6:\penalty0 044022, 2011.

\bibitem[Guan et~al.(2023)Guan, Page, Reich, Ventrucci, and
  Yang]{guan2023spectral}
Y.~Guan, G.~L. Page, B.~J. Reich, M.~Ventrucci, and S.~Yang.
\newblock Spectral adjustment for spatial confounding.
\newblock \emph{Biometrika}, 110:\penalty0 699--719, 2023.

\bibitem[Haar(1909)]{haar1909theorie}
A.~Haar.
\newblock \emph{Zur theorie der orthogonalen funktionensysteme}.
\newblock Georg-August-Universitat, Gottingen., 1909.

\bibitem[Huber(1964)]{Huber64}
P.~J. Huber.
\newblock {Robust Estimation of a Location Parameter}.
\newblock \emph{The Annals of Mathematical Statistics}, 35:\penalty0 73 -- 101,
  1964.

\bibitem[Hughes and Haran(2013)]{hughes2013dimension}
J.~Hughes and M.~Haran.
\newblock Dimension reduction and alleviation of confounding for spatial
  generalized linear mixed models.
\newblock \emph{Journal of the Royal Statistical Society Series B: Statistical
  Methodology}, 75:\penalty0 139--159, 2013.

\bibitem[Janzing et~al.(2009)Janzing, Peters, Mooij, and Sch\"{o}lkopf]{janzig}
D.~Janzing, J.~Peters, J.~Mooij, and B.~Sch\"{o}lkopf.
\newblock Identifying confounders using additive noise models.
\newblock In \emph{Proceedings of the Twenty-Fifth Conference on Uncertainty in
  Artificial Intelligence}, page 249–257, 2009.

\bibitem[Keller and Szpiro(2020)]{keller2020selecting}
J.~P. Keller and A.~A. Szpiro.
\newblock Selecting a scale for spatial confounding adjustment.
\newblock \emph{Journal of the Royal Statistical Society Series A: Statistics
  in Society}, 183:\penalty0 1121--1143, 2020.

\bibitem[Klivans et~al.(2018)Klivans, Kothari, and Meka]{klivans2018efficient}
A.~Klivans, P.~K. Kothari, and R.~Meka.
\newblock Efficient algorithms for outlier-robust regression.
\newblock In \emph{Conference On Learning Theory}, pages 1420--1430, 2018.

\bibitem[Leduc et~al.(2019)Leduc, Mailhot, Frigon, Martel, Ludwig, Brietzke,
  Gigu{\`e}re, Brissette, Turcotte, Braun, et~al.]{leduc2019climex}
M.~Leduc, A.~Mailhot, A.~Frigon, J.-L. Martel, R.~Ludwig, G.~B. Brietzke,
  M.~Gigu{\`e}re, F.~Brissette, R.~Turcotte, M.~Braun, et~al.
\newblock The climex project: A 50-member ensemble of climate change
  projections at 12-km resolution over europe and northeastern north america
  with the canadian regional climate model (crcm5).
\newblock \emph{Journal of Applied Meteorology and Climatology}, 58:\penalty0
  663--693, 2019.

\bibitem[Madden(1976)]{madden1976estimates}
R.~A. Madden.
\newblock Estimates of the natural variability of time-averaged sea-level
  pressure.
\newblock \emph{Monthly Weather Review}, 104:\penalty0 942--952, 1976.

\bibitem[Mahecha et~al.(2010)Mahecha, Reichstein, Carvalhais, Lasslop, Lange,
  Seneviratne, Vargas, Ammann, Arain, Cescatti, et~al.]{mahecha2010global}
M.~D. Mahecha, M.~Reichstein, N.~Carvalhais, G.~Lasslop, H.~Lange, S.~I.
  Seneviratne, R.~Vargas, C.~Ammann, M.~A. Arain, A.~Cescatti, et~al.
\newblock Global convergence in the temperature sensitivity of respiration at
  ecosystem level.
\newblock \emph{Science}, 329:\penalty0 838--840, 2010.

\bibitem[Marques et~al.(2022)Marques, Kneib, and Klein]{marques2022mitigating}
I.~Marques, T.~Kneib, and N.~Klein.
\newblock Mitigating spatial confounding by explicitly correlating gaussian
  random fields.
\newblock \emph{Environmetrics}, 33:\penalty0 2727, 2022.

\bibitem[Newey and West(1987)]{newey1986simple}
W.~K. Newey and K.~D. West.
\newblock A simple, positive semi-definite, heteroskedasticity and
  autocorrelation consistent covariance matrix.
\newblock \emph{Econometrica}, 55:\penalty0 703--708, 1987.

\bibitem[Paciorek(2010)]{paciorek2010importance}
C.~J. Paciorek.
\newblock The importance of scale for spatial-confounding bias and precision of
  spatial regression estimators.
\newblock \emph{Statistical Science: A Review Journal of the Institute of
  Mathematical Statistics}, 25:\penalty0 107, 2010.

\bibitem[Page et~al.(2017)Page, Liu, He, and Sun]{page2017estimation}
G.~L. Page, Y.~Liu, Z.~He, and D.~Sun.
\newblock Estimation and prediction in the presence of spatial confounding for
  spatial linear models.
\newblock \emph{Scandinavian Journal of Statistics}, 44:\penalty0 780--797,
  2017.

\bibitem[Pearl(2009)]{pearl2009causality}
J.~Pearl.
\newblock \emph{Causality}.
\newblock Cambridge University Press, 2009.

\bibitem[Pensia et~al.(2020)Pensia, Jog, and Loh]{pensia2020robust}
A.~Pensia, V.~Jog, and P.-L. Loh.
\newblock Robust regression with covariate filtering: Heavy tails and
  adversarial contamination.
\newblock \emph{arXiv preprint arXiv:2009.12976}, 2020.

\bibitem[Peters et~al.(2017)Peters, Janzing, and
  Sch{\"o}lkopf]{peters2017elements}
J.~Peters, D.~Janzing, and B.~Sch{\"o}lkopf.
\newblock \emph{Elements of Causal inference: Foundations and Learning
  Algorithms}.
\newblock The MIT Press, 2017.

\bibitem[Prates et~al.(2019)Prates, Assun{\c{c}}{\~a}o, and
  Rodrigues]{prates2019alleviating}
M.~O. Prates, R.~M. Assun{\c{c}}{\~a}o, and E.~C. Rodrigues.
\newblock Alleviating spatial confounding for areal data problems by displacing
  the geographical centroids.
\newblock \emph{Bayesian Analysis}, 14:\penalty0 623 -- 647, 2019.

\bibitem[Reich et~al.(2006)Reich, Hodges, and Zadnik]{reich2006effects}
B.~J. Reich, J.~S. Hodges, and V.~Zadnik.
\newblock Effects of residual smoothing on the posterior of the fixed effects
  in disease-mapping models.
\newblock \emph{Biometrics}, 62:\penalty0 1197--1206, 2006.

\bibitem[Reiers{\o}l(1945)]{reiersol1945confluence}
O.~Reiers{\o}l.
\newblock \emph{Confluence Analysis by Means of Instrumental Sets of
  Variables}.
\newblock PhD thesis, Almqvist \& Wiksell, 1945.

\bibitem[Rubin(2005)]{bb9b1ce1-c4cd-345c-b1f2-6bb227500876}
D.~B. Rubin.
\newblock Causal inference using potential outcomes: Design, modeling,
  decisions.
\newblock \emph{Journal of the American Statistical Association}, 100:\penalty0
  322--331, 2005.

\bibitem[Sasai and Fujisawa(2020)]{sasai2020robust}
T.~Sasai and H.~Fujisawa.
\newblock Robust estimation with lasso when outputs are adversarially
  contaminated.
\newblock \emph{arXiv preprint arXiv:2004.05990}, 2020.

\bibitem[Scheidegger et~al.(2023)Scheidegger, Guo, and
  B{\"u}hlmann]{scheidegger2023spectral}
C.~Scheidegger, Z.~Guo, and P.~B{\"u}hlmann.
\newblock Spectral deconfounding for high-dimensional sparse additive models.
\newblock \emph{arXiv preprint arXiv:2312.02860}, 2023.

\bibitem[Schneider and Kinter(1994)]{schneider1994examination}
E.~Schneider and J.~Kinter.
\newblock An examination of internally generated variability in long climate
  simulations.
\newblock \emph{Climate Dynamics}, 10:\penalty0 181--204, 1994.

\bibitem[Sippel et~al.(2019)Sippel, Meinshausen, Merrifield, Lehner,
  Pendergrass, Fischer, and Knutti]{sippel2019uncovering}
S.~Sippel, N.~Meinshausen, A.~Merrifield, F.~Lehner, A.~G. Pendergrass,
  E.~Fischer, and R.~Knutti.
\newblock Uncovering the forced climate response from a single ensemble member
  using statistical learning.
\newblock \emph{Journal of Climate}, 32:\penalty0 5677--5699, 2019.

\bibitem[Suggala et~al.(2019)Suggala, Bhatia, Ravikumar, and
  Jain]{suggala2019adaptive}
A.~S. Suggala, K.~Bhatia, P.~Ravikumar, and P.~Jain.
\newblock Adaptive hard thresholding for near-optimal consistent robust
  regression.
\newblock In \emph{Conference on Learning Theory}, pages 2892--2897, 2019.

\bibitem[Thaden and Kneib(2018)]{thaden2018structural}
H.~Thaden and T.~Kneib.
\newblock Structural equation models for dealing with spatial confounding.
\newblock \emph{The American Statistician}, 72:\penalty0 239--252, 2018.

\bibitem[Thams et~al.(2022)Thams, S{\o}ndergaard, Weichwald, and
  Peters]{thams2022identifying}
N.~Thams, R.~S{\o}ndergaard, S.~Weichwald, and J.~Peters.
\newblock Identifying causal effects using instrumental time series: Nuisance
  {IV} and correcting for the past.
\newblock \emph{arXiv preprint arXiv:2203.06056}, 2022.

\bibitem[Vershynin(2018)]{vershynin2018high}
R.~Vershynin.
\newblock \emph{High-Dimensional Probability: An Introduction with Applications
  in Data Science}.
\newblock Cambridge University Press, 2018.

\bibitem[Wallace et~al.(1995)Wallace, Zhang, and Renwick]{wallace1995dynamic}
J.~M. Wallace, Y.~Zhang, and J.~A. Renwick.
\newblock Dynamic contribution to hemispheric mean temperature trends.
\newblock \emph{Science}, 270:\penalty0 780--783, 1995.

\bibitem[Wright(1928)]{wright1928tariff}
P.~Wright.
\newblock \emph{The Tariff on Animal and Vegetable Oils}.
\newblock Macmillan, 1928.

\end{thebibliography}

\appendix

\section{Visualization of \algon}
\label{app:figure}
Figure~\ref{fig:algon_vis} provides another visualization of the proposed procedure \algon.
\begin{figure}[ht]
    \centering
    \begin{tikzpicture}
        \node[inner sep=0pt] (1) at (0,0)
        {\includegraphics[width=.25\textwidth]{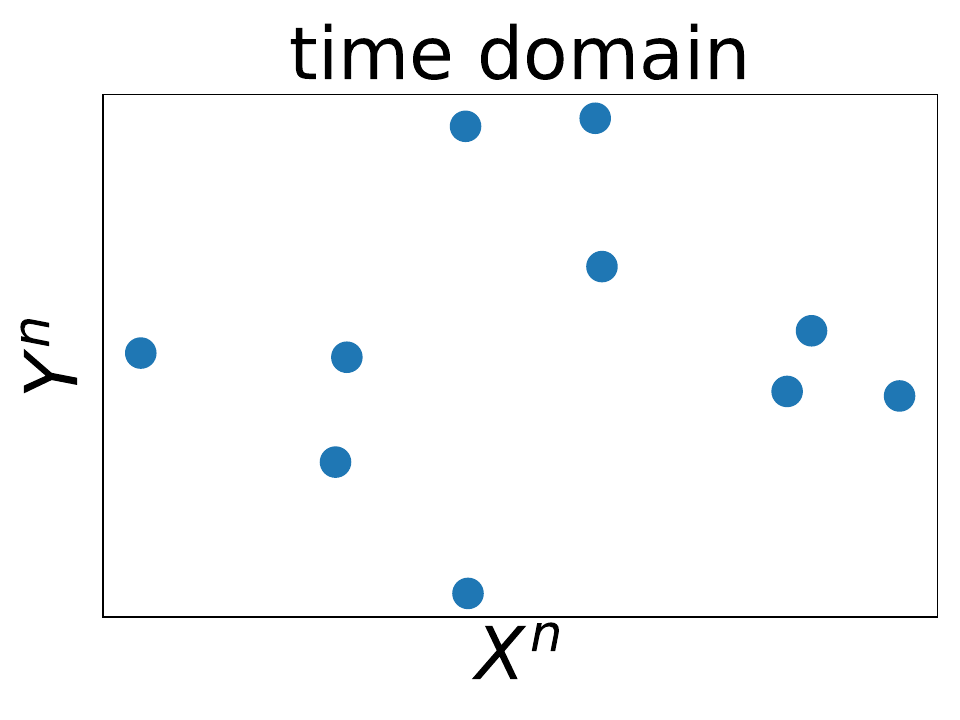}};
        \node[inner sep=0pt] (2) at (6,0)
        {\includegraphics[width=.25\textwidth]{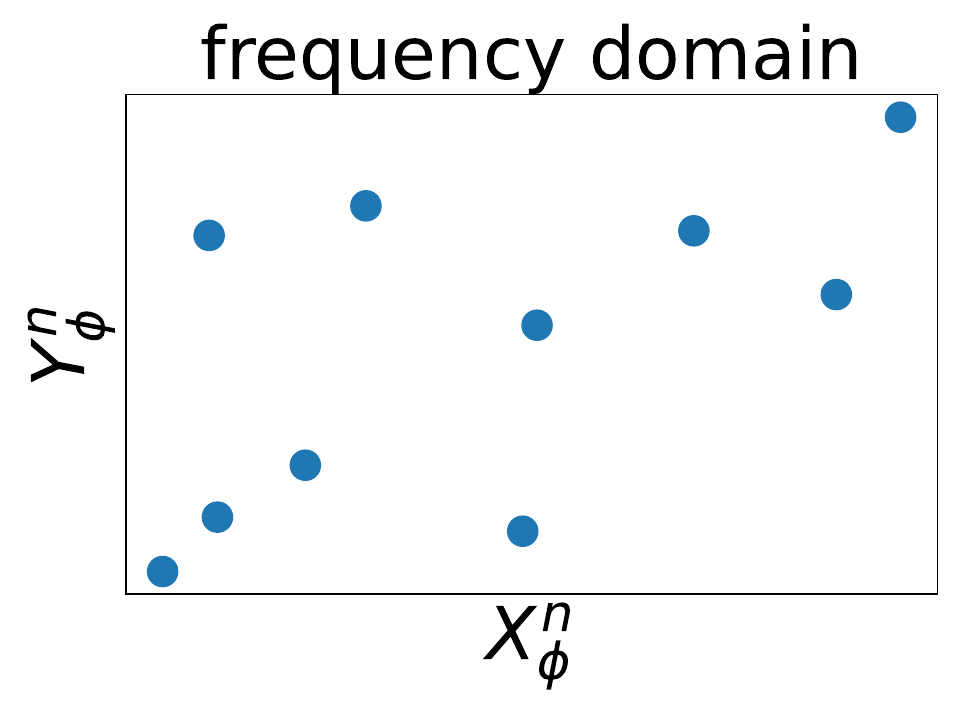}};
        \node[inner sep=0pt] (3) at (12,0)
        {\includegraphics[width=.25\textwidth]{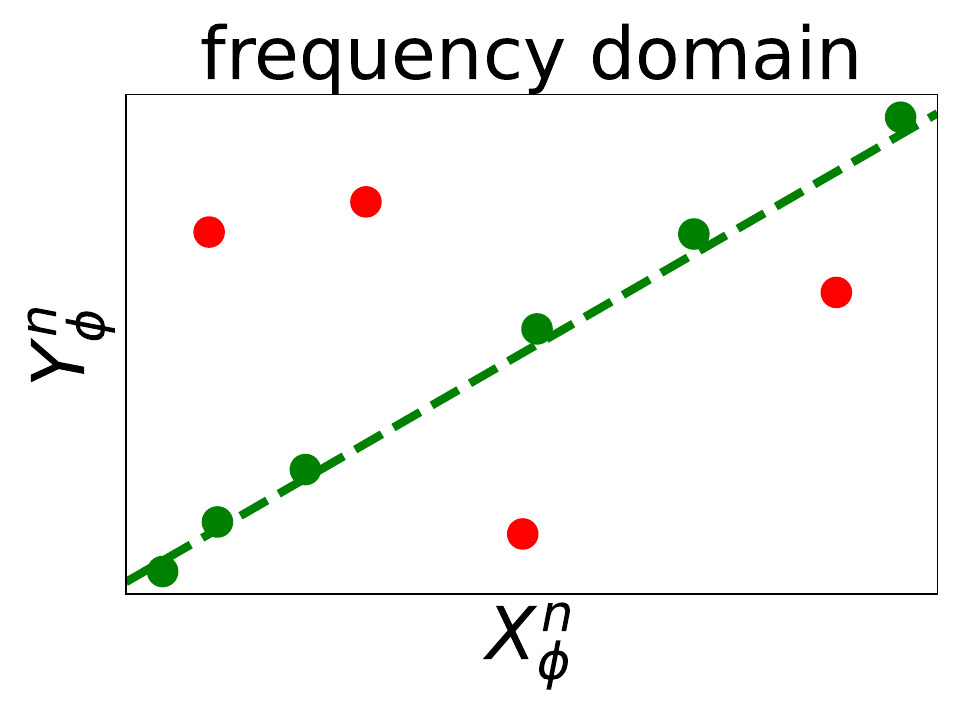}};
        
        \draw[->, thick, -Latex] (1) to node[above,sloped] {$T^{\phi,n}$} (2);
        \draw[->, thick, -Latex] (2) to node[above,sloped] {robust} node[below,sloped] {regression} (3);
    \end{tikzpicture}
    \caption{
    Visualization of \algon
    (Section~\ref{sec:sparse_deconfounding_robust}): in the time domain (left), the processes are confounded, so estimators based on least squares will generally be biased. Assuming that the confounding is sparse in the frequency domain (middle), we can apply robust regression methods to estimate the causal effect (right). We prove that \algon is consistent under weak assumptions if the confounding is sparse; see Section~\ref{sec:theoretical_g_an}.
    }
    \label{fig:algon_vis}
\end{figure}

\section{Comment on Setting~\ref{set:robust_improved}}
\label{sec:more_general}
Setting~\ref{set:robust_improved} 
is more general than the additive structure may initially suggest.
To see this, 
let $\tilde{X} = (\tilde{X}_t)_{t \in [0,T]}$ be a stochastic process and let $\epsilon = (\epsilon_t)_{t \in [0,T]}$ be any noise process; in particular, $\epsilon$ might depend on $\tilde{X}$.
Assume that
the stochastic process $\tilde{Y} = (\tilde{Y}_t)_{t \in [0,T]}$ 
satisfies
\begin{equation}
\label{eq:hhde}
    \tilde{Y}_t \coloneqq \tilde{\beta} \tilde{X}_t + \epsilon_t,
\end{equation} 
where $\tilde{\beta} \in \mathbb{R}^d$ denotes the total causal effect of $\tilde{X}_t$ on $\tilde{Y}_t$.
Without loss of generality, there exist
random vectors $\tilde{U} = \{\tilde{U}_t\}_{t \in [0,T]}$, an independent stochastic process $\tau = (\tau_t)_{t \in [0,T]}$, and measurable function $(h_t)_{t \in [0,T]}$ such that
\begin{align*}
    \tilde{X}_t &\coloneqq h_t(\tilde{U}_t, \tau),\\
    \tilde{Y}_t &\coloneqq \tilde{\beta} \tilde{X}_t + \tilde{U}_t
\end{align*}
and $\tilde{U}$ and $\tau$ are independent.
This statement follows when defining
$\tilde{U}_t := \epsilon_t$ and choosing the equation $h_t$ and $\tau$ accordingly
\citep[see][Proposition~4.1]{peters2017elements}.
In this structural causal model, $\tilde{\beta}$ is the causal effect from $\tilde{X}_t$ to $\tilde{Y}_t$.
Since we allow for a degenerate noise process,
Setting~\ref{set:robust_improved} thus includes \eqref{eq:hhde} as a special case.

\section{Additional Definitions}
\label{sec:defs}
\begin{definition}[Cosine basis]
\label{def:cosine_basis}
    Let $T \in \mathbb{R}_{>0}$ be a fixed constant. 
    Define, for all $k \in \mathbb{N}$, the function $\phi_k: [0,T] \to \mathbb{R}$ by
    \begin{equation*}
        x \mapsto
        \begin{cases}
            1 & \text{if $x=0$ and}\\
            \sqrt{2}\cos(x \pi (k + 1/2)) & \text{otherwise.}
        \end{cases}
    \end{equation*}
    We call $\phi \coloneqq (\phi_k)_{k \in \mathbb{N}}$ the \emph{cosine basis}.
\end{definition}
By definition, the cosine basis is right-continuous and its discretization
is also called the inverse DCT \citep{ahmed1974discrete}.

\begin{definition}[Haar basis \citep{haar1909theorie}]
\label{def:haar_basis}
    Let $T \in \mathbb{R}_{>0}$ be a fixed constant. Define the function $\phi^{\mathrm{H}}_{0}: [0,T] \to \mathbb{R}$ by
    \begin{equation*}
        x \mapsto
        \begin{cases}
            1 & \text{if $x \leq T/2$}\\
            -1 & \text{if $x > T/2$}.
        \end{cases}
    \end{equation*}
    Define, for all $k \in \mathbb{N}$, the function $\phi^{\mathrm{H}}_{k}: [0,T] \to \mathbb{R}$ by
    \begin{equation*}
        x \mapsto 2^{-k/2} \phi^{\mathrm{H}}_{0}(2^k x).
    \end{equation*}
    We call $\phi \coloneqq (\phi_k)_{k \in \mathbb{N}}$ the \emph{Haar basis}.
\end{definition}
By definition, the Haar basis is right-continuous.

\begin{definition}[Bachmann–Landau notation]
\label{def:o_not_1}
    Let $f,g: \mathbb{N} \to \mathbb{R}$. We say $|f|$ is \emph{bounded from above by g asymptotically} and write $f(n) \in \mathcal{O}(g(n))$ if 
    \begin{equation*}
        \limsup_{n \to \infty} \frac{f(n)}{g(n)} < \infty.
    \end{equation*}
    We say $f$ is \emph{bounded from below by g asymptotically } and write  $f(n) \in \Omega(g(n))$ if 
    \begin{equation*}
        \liminf_{n \to \infty} \frac{f(n)}{g(n)} > 0.
    \end{equation*}
\end{definition}

\begin{definition}
\label{def:o_not}
    Let 
    $X_1, X_2, \ldots$ be a sequence of random variables and $g: \mathbb{N} \to \mathbb{R}$ a real-valued function. We say that 
    we have 
    \begin{equation*}
        X_n \in \mathcal{O}_{\mathbb{P}}(g(n))
    \end{equation*}
    if for all $\delta > 0$ there exists a real-valued function $f_{\delta}: \mathbb{N} \to \mathbb{R}$ such that 
    \begin{equation*}
        \text{for all } n \in \mathbb{N}: \qquad 
        \mathbb{P}\left[ X_n \leq f_{\delta}(n) \right] \geq 1 - \delta
    \end{equation*}
    and
    \begin{equation*}
        f_{\delta}(n) \in \mathcal{O}(g(n)).
    \end{equation*}
    We define  $X_n \in o_{\mathbb{P}}(g(n))$ 
    and $X_n \in \Omega_{\mathbb{P}}(g(n))$ 
    analogously.
\end{definition}

\section{Proofs}
\subsection{Lemmas}
\begin{lemma}[Theorem 6.3.2 of \cite{vershynin2018high}]
\label{thm:concentration}
    Let $\mu$ be a sub-Gaussian distribution with zero mean and unit variance. There exists a constant $K>0$ such that for all $n \in \mathbb{N}$ the following holds: Let $\epsilon_1, \dots, \epsilon_n$ be i.i.d~distributed with respect to $\mu$ and define $\epsilon^n \coloneqq (\epsilon_1, \dots, \epsilon_n)$. Then for all $A \in \mathbb{R}^{m \times n}$ 
    and $t \geq 0$ we have
    that
    \begin{align*}
        \mathbb{P}\left( \Big \vert \norm{A\epsilon^n}_2 - \norm{A}_{\mathrm{F}} \Big \vert \geq t \right) \leq 2\exp\left( -\frac{t^2}{K\norm{A}^2_2} \right).
    \end{align*}
\end{lemma}

\subsection{Proof of Theorem~\ref{lem:ransac_gaussian}}
\label{app:proofThm31}

We first proof a result for general noise variables:
\begin{theorem}
\label{lem:ransac}
    Assume Setting~\ref{set:outlier} with $d = 1$ %
    and that $\B \neq \emptyset$.
    Define for $S \in \sampleset$ and for $U_n \in \B$
    \begin{equation}
    \label{eq:ass_ransac_2}
        \eta_1(S) \coloneqq \frac{|X_S^{\top} \epsilon_S|}{\norm{X_S}_2^2}
    \end{equation}
    and
    \begin{align}
    \label{eq:ass_ransac}
    \begin{split}
        \eta_2^2(S, U_n) &\coloneqq \frac{1}{\norm{X_{S \setminus \outliers}}_2^2} \Bigg(\frac{|{S}|}{|U_n|} \norm{\epsilon_{U_n}}_2^2 - \norm{\epsilon_{S}}_2^2 - 2 \epsilon_{S}^{\top} o_{S}\\
        &\quad + \frac{(X_{S}^{\top} \epsilon_{S})^2}{\norm{X_{S}}_2^2} + 2\frac{\left|X_{S}^{\top} \epsilon_{S} \right| \left| X_{S}^{\top} o_{S} \right|}{\norm{X_{S}}_2^2} - \frac{|{S}|(X_{U_n}^{\top} \epsilon_{U_n})^2}{|U_n|\norm{X_{U_n}}_2^2}\Bigg).
    \end{split}
    \end{align}
    Then
    \begin{equation*}
        \left| \hat{\beta}^n_{\bfs}(\sampleset) - \beta \right| \leq \max_{S \in \sampleset} \min_{U_n \in \B} \left(\eta_1(S) + \eta_2(S, U_n) \right).
    \end{equation*}
\end{theorem}

\begin{lemma}
\label{lem:pred_err}
    Assume Setting~\ref{set:outlier} with $d = 1$. Let $S \in \mathcal{P}(\{1, \dots, n\}) \setminus \{\emptyset\}$ and assume that $\norm{X_S}_2 > 0$. The squared prediction error of the \ols estimator when using the data with indices $S$ is given by
    \begin{align*}
        \norm{\hat{\beta}_{\ols}^S X_S - Y_S}_2^2 
        &= \norm{\epsilon_S}_2^2 + \norm{o_S}_2^2 + 2\epsilon_S^{\top} o_S - \left( \frac{X_S^{\top}
        \epsilon_S}{\norm{X_S}_2}
        + \frac{X_S^{\top}o_S}{\norm{X_S}_2} 
        \right)^2.
    \end{align*}
\end{lemma}
\begin{proof}
It holds that
    \begin{align*}
        \norm{\hat{\beta}_{\ols}^S X_S - Y_S}_2^2
        &= \norm{\frac{X_S^{\top} (\epsilon_{S} + o_S) X_S}{\norm{X_{S}}_2^2} - \epsilon_S - o_S}_2^2\\
        &= \norm{\epsilon_S + o_S}_2^2 - \left( \frac{X_S^{\top}}{\norm{X_S}_2} (\epsilon_S + o_S) \right )^2\\
        & = \norm{\epsilon_S}_2^2 + \norm{o_S}_2^2 + 2\epsilon_S^{\top} o_S - \left( \frac{X_S^{\top}
        \epsilon_S}{\norm{X_S}_2}
        + \frac{X_S^{\top}o_S}{\norm{X_S}_2}
        \right)^2.
    \end{align*}
\end{proof}

We are now able to prove 
Theorem~\ref{lem:ransac}.
Let
\begin{equation}
\label{eq:s_star}
    S^* \in \argmin_{S \in \sampleset} \frac{1}{|S|} \norm{Y_S - X_S \hat{\beta}^S_{\ols}(X,Y)}^2_2
\end{equation}
be the set satisfying %
$\hat{\beta}^{S^*}_{\ols} = \hat{\beta}^n_{\bfs}(\sampleset)$.
\begin{proof}[Proof of Theorem~\ref{lem:ransac}]
    Assume $X_{S^*}^{\top}X_{S^*}$ is invertible (otherwise the bound is void). %
    By the definition of $S^*$ (see \eqref{eq:s_star}),
    \begin{equation*}
        \frac{1}{|S^*|}\norm{\hat{\beta}_{\ols}^{S^*} X_{S^*} - Y_{S^*}}_2^2 \leq \frac{1}{|U_n|}\norm{\hat{\beta}_{\ols}^{U_n} X_{U_n} - Y_{U_n}}^2_2.
    \end{equation*}
    This implies by Lemma~\ref{lem:pred_err} 
    that
    \begin{align*}
        \frac{1}{|S^*|} \left( \norm{\epsilon_{S^*}}_2^2 + \norm{o_{S^*}}_2^2 + 2\epsilon_{S^*}^{\top} o_{S^*} - \left( \frac{X_{S^*}^{\top}\epsilon_{S^*}}{\norm{X_{S^*}}_2} + \frac{X_{S^*}^{\top}o_{S^*}}{\norm{X_{S^*}}_2} \right)^2 \right)
        &= \frac{1}{|S^*|} \norm{\hat{\beta}_{\ols}^{S^*} X_{S^*} - Y_{S^*}}_2^2\\
        &\leq \frac{1}{|U_n|} \norm{\hat{\beta}_{\ols}^{U_n} X_{U_n} - Y_{U_n}}^2_2\\
        &\leq \frac{\norm{\epsilon_{U_n}}_2^2}{|U_n|} - \frac{(X_{U_n}^{\top} \epsilon_{U_n})^2}{|U_n|\norm{X_{U_n}}_2^2},
    \end{align*}
    using $o_{U_n} = 0$ in the last inequality.
    Rearranging
    gives
    \begin{align}
    \label{eq:intermediate}
    \begin{split}
        \norm{o_{S^*}}^2_2 - \left( \frac{X_{S^*}^{\top} o_{S^*}}{\norm{X_{S^*}}_2} \right)^2 &\leq \frac{|{S^*}|}{|U_n|} \norm{\epsilon_{U_n}}_2^2 - \norm{\epsilon_{S^*}}_2^2 - 2 \epsilon_{S^*}^{\top} o_{S^*} + \\
        &\quad\frac{(X_{S^*}^{\top} \epsilon_{S^*})^2}{\norm{X_{S^*}}_2^2} + 2\frac{\left|X_{S^*}^{\top} \epsilon_{S^*} \right| \left| X_{S^*}^{\top} o_{S^*} \right|}{\norm{X_{S^*}}_2^2} - \frac{|{S^*}|(X_{U_n}^{\top} \epsilon_{U_n})^2}{|U_n|\norm{X_{U_n}}_2^2}\\
        &= \eta_2^2(S^*, U_n) \norm{X_{{S^*} \setminus \outliers}}_2^2.
    \end{split}
    \end{align}
    By Cauchy-Schwarz and the fact that $o_i = 0$ for all $i \notin \outliers$ we have
    \begin{align}
    \label{eq:intermediate_2}
    \begin{split}
        \norm{o_{S^*}}^2_2  \frac{\norm{X_{S^* \setminus \outliers}}^2_2}{\norm{X_{S^*}}^2_2}
        &= \norm{o_{S^*}}^2_2 \left( 1 - \frac{\norm{X_{S^* \cap \outliers}}^2_2}{\norm{X_{S^*}}^2_2} \right) \\
        &\leq \norm{o_{S^*}}^2_2 \left( 1 - \frac{(X_{S^*\cap \outliers}^{\top} \; o_{S^*\cap \outliers})^2}{\norm{X_{S^*}}^2_2 \norm{o_{S^*\cap \outliers}}^2_2 } \right) \\
        &= \norm{o_{S^*}}^2_2 \left( 1 - \frac{(X_{S^*}^{\top} o_{S^*})^2}{\norm{X_{S^*}}^2_2 \norm{o_{S^*}}^2_2 } \right) \\
        &= \norm{o_{S^*}}^2_2 - \left( \frac{X_{S^*}^{\top} o_{S^*}}{\norm{X_{S^*}}_2} \right)^2.
    \end{split}
    \end{align}
Combining~\eqref{eq:intermediate} and~\eqref{eq:intermediate_2} yields
    \begin{equation*}
        \frac{\norm{o_{S^*}}_2}{\norm{X_{S^*}}_2} \leq \eta_2(S^*, U_n).
    \end{equation*}
    By~\eqref{eq:s_star}
    \begin{align*}
        \left|\hat{\beta}^n_{\bfs}(\sampleset) - \beta \right| &\leq \frac{|X_{S^*}^{\top} \epsilon_{S^*}|}{\norm{X_{S^*}}_2^2} + \frac{|X_{S^*}^{\top} o_{S^*}|}{\norm{X_{S^*}}_2^2}\\
        &\leq \frac{|X_{S^*}^{\top} \epsilon_{S^*}|}{\norm{X_{S^*}}_2^2} + \frac{\norm{o_{S^*}}_2}{\norm{X_{S^*}}_2}\\
        &\leq \eta_1(S^*) + \eta_2(S^*, U_n).
    \end{align*}
    Since this holds for all $U_n \in \B$,
    we have
    $$
            \left|\hat{\beta}^n_{\bfs}(\sampleset) - \beta \right|
            \leq
                \min_{U_n \in \B} \left(\eta_1(S^*) + \eta_2(S^*, U_n) \right)
                \leq
    \max_{S \in \sampleset} \min_{U_n \in \B} \left(\eta_1(S) + \eta_2(S, U_n) \right).
    $$
\end{proof}

\begin{proof}[Proof of Theorem~\ref{lem:ransac_gaussian}]
    Let $U_n \in \B$. By the proof of Theorem~\ref{lem:ransac}, specifically \eqref{eq:intermediate} and \eqref{eq:intermediate_2}, we have that
    \begin{align} \label{ineq:thm31}
        \norm{o_{S^*}}^2_2  \frac{\norm{X_{S^* \setminus \outliers}}^2_2}{\norm{X_{S^*}}^2_2} &\leq \frac{|{S^*}|}{|U_n|} \norm{\epsilon_{U_n}}_2^2 - \norm{\epsilon_{S^*}}_2^2 - 2 \epsilon_{S^*}^{\top} o_{S^*} \nonumber \\
        &\quad + \frac{(X_{S^*}^{\top} \epsilon_{S^*})^2}{\norm{X_{S^*}}_2^2} + 2\frac{\left|X_{S^*}^{\top} \epsilon_{S^*} \right| \left| X_{S^*}^{\top} o_{S^*} \right|}{\norm{X_{S^*}}_2^2} - \frac{|{S^*}|(X_{U_n}^{\top} \epsilon_{U_n})^2}{|U_n|\norm{X_{U_n}}_2^2}.
    \end{align}
    We bound the terms on the RHS.
    It holds that for all $S \in \mathcal{P}(\{1, \dots, n\}) \setminus \{\emptyset\}$
    \begin{align*}
        &X_S^{\top} \epsilon_S \sim \mathcal{N}\left( 0, \sigma^2 \norm{X_S}_2^2 \right)\\
        \text{and} \quad &o_S^{\top} \epsilon_S \sim \mathcal{N}\left( 0, \sigma^2 \norm{o_S}_2^2 \right).
    \end{align*}
    Therefore, by the Gaussian concentration inequality and the union bound,
    \begin{equation} \label{ineq:thm312}
        \mathbb{P}\left[ \forall S \in \sampleset, |X_S^{\top} \epsilon_S| \leq \sigma {\norm{X_S}_2} \sqrt{2\log(2|\sampleset|/\delta)} \right] \geq 1 - \delta
    \end{equation}
    and
    \begin{equation*}
        \mathbb{P}\left[ \forall S \in \sampleset, |o_S^{\top} \epsilon_S| \leq \sigma {\norm{o_S}_2} \sqrt{2\log(2|\sampleset|/\delta)} \right] \geq 1 - \delta.
    \end{equation*}
    By Lemma~\ref{thm:concentration}
    there exists $K >0$ such that for all $t \geq 0$ we have
    \begin{equation*}
        \mathbb{P}\left[ \forall S \in \sampleset, \left| \norm{\epsilon_S}_2 - \sigma \sqrt{|S|} \right| \geq t \right] \leq 2|\sampleset| \exp\left( -\frac{t^2}{K \sigma^2} \right)
    \end{equation*}
    and therefore
    \begin{equation*}
        \mathbb{P}\left[ \forall S \in \sampleset, \left| \norm{\epsilon_S}_2 - \sigma \sqrt{|S|} \right| \leq \sigma \sqrt{K\log(2|\sampleset|/\delta)} \right] \geq 1 - \delta.
    \end{equation*}
    This yields with probability at least $1 - \delta$ jointly for all $S \in \sampleset$
    \begin{align*}
        \frac{|S|}{|U_n|} \norm{\epsilon_{U_n}}_2^2 - \norm{\epsilon_S}_2^2 &\leq 
        \frac{|S|}{{|U_n|}} \left( \sigma \sqrt{|U_n|} + \sigma \sqrt{K\log(2|\sampleset|/\delta)}\right)^2 -\left( \sigma \sqrt{|S|} - \sigma \sqrt{K\log(2|\sampleset|/\delta)}\right)^2\\
        & = \alpha_1(S, U_n, \delta).
    \end{align*}
    Using the three inequalities above (and Cauchy-Schwarz for the term $|X_{S^*}^{\top} o_{S^*}|$)
    to bound the terms of the RHS of~\eqref{ineq:thm31}, 
    we therefore have with probability at least $1-3\delta$
    \begin{align*}
        \frac{\norm{o_{S^*}}^2_2}{  \norm{X_{S^*}}^2_2} &\leq \frac{\alpha_1(S^*, U_n, \delta)}{\norm{X_{{S^*} \setminus \outliers}}^2_2} + 4 \frac{\sigma \norm{o_{S^*}}_2 \sqrt{2\log(2|\sampleset|/\delta)}}{\norm{X_{{S^*} \setminus \outliers}}_2^2} + \left(\frac{\sigma \sqrt{2\log(2|\sampleset|/\delta)}}{\norm{X_{{S^*} \setminus \outliers}}_2}\right)^2.
    \end{align*}
    Solving the quadratic equation with respect to $\norm{o_{S^*}}_2$ yields
    \begin{align}
        \frac{\norm{o_{S^*}}_2}{  \norm{X_{S^*}}_2} &\leq
        \frac{4 \sigma \norm{X_{{S^*}}}_2 \sqrt{2\log(2|\sampleset|/\delta)}}{\norm{X_{{S^*} \setminus \outliers}}_2^2}
        + \frac{\sqrt{\alpha_1(S^*, U_n, \delta)}}{\norm{X_{{S^*} \setminus \outliers}}_2}
        +  \frac{\sigma \sqrt{2\log(2|\sampleset|/\delta)}}{\norm{X_{{S^*} \setminus \outliers}}_2} \nonumber\\
        &= \alpha(S^*, U_n , \delta) \label{ineq:thm314}
    \end{align}
As shown in the proof of Theorem~\ref{lem:ransac}, 
we have
$$
        \left|\hat{\beta}^n_{\bfs}(\sampleset) - \beta \right| 
        \leq \frac{|X_{S^*}^{\top} \epsilon_{S^*}|}{\norm{X_{S^*}}_2^2} + \frac{\norm{o_{S^*}}_2}{\norm{X_{S^*}}_2}.
$$
Using~\eqref{ineq:thm312} and~\eqref{ineq:thm314},
we thus obtain
with probability at least $1-3\delta$
    \begin{equation*}
        \left| \hat{\beta}^n_{\bfs}(\sampleset) - \beta \right| \leq \min_{U_n \in \B} \left\{ \alpha(S^*, U_n, \delta) + \frac{\sigma \sqrt{2\log(2|\sampleset|/\delta)}}{\norm{X_{S^*}}_2} \right\}.
    \end{equation*}
and therefore the desired result.
\end{proof}

\subsection{Proof of Lemma~\ref{lem:tor}} \label{app:proofoflem32}

    The number of distinct sets $S$ that \torrent can output is trivially bounded by $\binom{n}{\lfloor a \rfloor}$.
    Furthermore, it holds for all $t$ that
    \begin{align}
    \label{eq:74738478}
        \norm{X_{S_{t+1}}(\beta - \hat{\beta}^{t+1}_{\tor}) + \epsilon_{S_{t+1}} + o_{S_{t+1}}}_2
        &\leq \norm{X_{S_{t}}(\beta - \hat{\beta}^{t+1}_{\tor}) + \epsilon_{S_{t}} + o_{S_{t}}}_2\\
        &\leq \norm{X_{S_{t}}(\beta - \hat{\beta}^{t}_{\tor}) + \epsilon_{S_{t}} + o_{S_{t}}}_2 \label{eq:89898989898}.
    \end{align}
    The first inequality holds because of the hard thresholding step, while the second holds because $\hat{\beta}_{\tor}^{t+1}$ is the least squares estimate for data $X_{S_{t}}$. This implies that \torrent converges after finitely many steps.

\subsection{Proof of Theorem~\ref{thm:tor}} \label{app:proofThm33}
    Assume that $X_{S_{t-1}}^{\top}X_{S_{t-1}}$ is invertible (otherwise, the assumptions for $\eta$ are not satisfied).
    We begin by observing that the optimality of the model $\hat{\beta}^{t}_{\tor}$ on the active set $S_{t}$ ensures
    \begin{align}
    \label{eq:56456}
    \begin{split}
        \norm{\hat{\beta}^{t}_{\tor} - \beta} _2 &= \norm{ \left(X_{S_{t-1}}^{\top} X_{S_{t-1}} \right)^{-1} \left(X_{S_{t-1}}^{\top} \epsilon_{S_{t-1}} + X_{S_{t-1}}^{\top} o_{S_{t-1}} \right) }_2 \\
        &\leq \norm{ \left(X_{S_{t-1}}^{\top} X_{S_{t-1}} \right)^{-1} X_{S_{t-1}}^{\top} \epsilon_{S_{t}} }_2 + \norm{ \left(X_{S_{t-1}}^{\top} X_{S_{t-1}} \right)^{-1} X_{S_{t-1}}^{\top}}_2  \norm{o_{S_{t-1}} }_2.
    \end{split}
    \end{align}
    The hard thresholding step guarantees that
    \begin{align}
        \begin{split}
        \label{eq:783487437}
            \left\|X_{S_{t}}\left(\beta-\hat{\beta}^{t}_{\tor}\right)+\epsilon_{S_{t}}+o_{S_{t}}\right\|_2^2 & =\left\|Y_{S_{t}}-X_{S_{t}} \hat{\beta}^{t}_{\tor}\right\|_2^2 \\
            & \leq\left\|Y_{\inliers}-X_{\inliers} \hat{\beta}^{t}_{\tor}\right\|^2_2 \\
            & =\left\|X_{\inliers}\left(\beta-\hat{\beta}^{t}_{\tor}\right)+\epsilon_{\inliers}\right\|_2^2 .
        \end{split}
    \end{align}
    Let $H_{t} \coloneqq S_{t} \backslash \inliers$ and $M_{t} \coloneqq \inliers \backslash S_{t}$. It then follows from \eqref{eq:783487437}
    $$
    \left\|X_{H_{t}}\left(\beta-\hat{\beta}^{t}_{\tor}\right)+\epsilon_{H_{t}}+o_{H_{t}}\right\|_2 \leq\left\|X_{M_{t}}\left(\beta-\hat{\beta}^{t}_{\tor}\right)+\epsilon_{M_{t}}\right\|_2 .
    $$
    Let $V_{t} \coloneqq H_{t} \cup M_{t}$ then $V_{t} = (S_{t} \cup \inliers) \setminus (\inliers \cap S_{t}) = V(S_t)$.
    An application of the triangle inequality and the fact that $\left\|o_{H_{t}}\right\|_2=\left\|o_{S_{t}}\right\|_2$ gives us
    \begin{align}
    \begin{split}
    \label{eq:tor_55}
        \left\|o_{S_{t}}\right\|_2 & \leq\left\|X_{M_{t}}\left(\beta-\hat{\beta}^{t}_{\tor}\right)\right\|_2+\left\|X_{H_{t}}\left(\beta-\hat{\beta}^{t}_{\tor}\right)\right\|_2+\left\|\epsilon_{H_{t}}\right\|_2+\left\|\epsilon_{M_{t}}\right\|_2 \\
        & \leq \sqrt{2} \norm{X_{V_{t}}}_2 \left\|\beta-\hat{\beta}^{t}_{\tor}\right\|_2+\sqrt{2}\|\epsilon_{V_{t}}\|_2 \\
        & \leq \sqrt{2} \norm{X_{V_{t}}}_2 \left( \norm{ \left(X_{S_{t-1}}^{\top} X_{S_{t-1}} \right)^{-1} X_{S_{t-1}}^{\top} \epsilon_{S_{t-1}} }_2 + \norm{ \left(X_{S_{t-1}}^{\top} X_{S_{t-1}} \right)^{-1} X_{S_{t-1}}^{\top}}_2  \norm{o_{S_{t-1}} }_2 \right) \\
        &\quad + \sqrt{2}\|\epsilon_{V_{t}}\|_2\\
        &= \sqrt{2}\norm{X_{V_{t}}}_2 \norm{\left(X_{S_{t-1}}^{\top} X_{S_{t-1}} \right)^{-1}X_{S_{t-1}}^{\top}}_2 \norm{o_{S_{t-1}}}_2 + \sqrt{2} \norm{X_{V_{t}}}_2  \norm{ \left(X_{S_{t-1}}^{\top} X_{S_{t-1}} \right)^{-1} X_{S_{t-1}}^{\top} \epsilon_{S_{t-1}} }_2 \\
        &\quad + \sqrt{2}\|\epsilon_{V_{t}}\|_2.
    \end{split}
    \end{align}
    Here the second inequality follows since for all $a, b \in \mathbb{R}$ we have $(|a| + |b|)^2 \leq 2 (a^2 + b^2)$.
    Let $t$ be the iteration where \torrent has converged (see Lemma~\ref{lem:tor}). 
    We prove that we can assume that $S_t = S_{t-1}$.
First, assume that $\hat{\beta}^{t}_{\tor} = \hat{\beta}^{t-1}_{\tor}$. Then, by the (deterministic) hard thresholding step of Algorithm~\ref{alg:torrent} we have $S_{t} = S_{t-1}$.
    Next, assume $\hat{\beta}^{t}_{\tor} \neq \hat{\beta}^{t-1}_{\tor}$. If $X_{S_{t-1}}^T X_{S_{t-1}}$ is not invertible, the assumptions on $\eta$ are not satisfied. 
        If $X_{S_{t-1}}^T X_{S_{t-1}}$ is invertible, then $\hat{\beta}^{t}_{\tor}$ is the unique minimum
        and, therefore, \eqref{eq:74738478} is a strict inequality. Together with~\eqref{eq:89898989898} this implies that the algorithm did not stop, which is a contraction to the assumption that $S_t$ is the output of \torrent. Without loss of generality, we can therefore assume that $S_t = S_{t-1}$.
    
    Since for all $A \in \mathbb{R}^{n \times d}$  we have $\norm{A}_2 = \sqrt{\lambda_{\mathrm{max}}(A^{\top}A)} = \sqrt{\lambda_{\mathrm{max}}(AA^{\top})} = (\lambda_{\mathrm{min}}((AA^{\top})^{-1}))^{-1/2}$ and by assumption we have
    \begin{equation*}
        \eta = \max_{S \subseteq \{1,\dots,n\} \text{ s.t. } |S| = a_n}\frac{\norm{X_{V(S)}}_2}{\sqrt{\lambda_{\mathrm{min}}\left( X_{S}^T X_{S}\right)}}
        =\max_{S \subseteq \{1,\dots,n\} \text{ s.t. } |S| = a_n}\norm{X_{V(S)}}_2\norm{\left(X_{S}^{\top} X_{S} \right)^{-1}X_{S}^{\top}}_2 < \frac{1}{\sqrt{2}}.
    \end{equation*}
    Together with $S_t = S_{t-1}$ and \eqref{eq:tor_55} this yields
    \begin{align*}
        \left\|o_{S_{t}}\right\|_2 \leq \sqrt{2} \eta \norm{o_{S_{t}}}_2 + \sqrt{2} \norm{X_{V_{t}}}_2  \norm{ \left(X_{S_{t}}^{\top} X_{S_{t}} \right)^{-1} X_{S_{t}}^{\top} \epsilon_{S_{t}} }_2 + \sqrt{2}\|\epsilon_{V_{t}}\|_2
    \end{align*}
    and therefore
    \begin{equation*}
        \left\|o_{S_{t}}\right\|_2 \leq \frac{\sqrt{2} \norm{X_{V_{t}}}_2  \norm{ \left(X_{S_{t}}^{\top} X_{S_{t}} \right)^{-1} X_{S_{t}}^{\top} \epsilon_{S_{t}} }_2 + \sqrt{2}\|\epsilon_{V_{t}}\|_2}{1-\sqrt{2}\eta}.
    \end{equation*}
    Together with~\eqref{eq:56456} this yields the following result.
    \begin{align*}
        \norm{\hat{\beta}^{n, a_n}_{\tor} - \beta} _2 = \norm{\hat{\beta}^{t}_{\tor} - \beta} _2 &\leq \norm{ \left(X_{S_{t}}^{\top} X_{S_{t}} \right)^{-1} X_{S_{t}}^{\top} \epsilon_{S_{t}} }_2 \\
        & \quad + \norm{ \left(X_{S_{t}}^{\top} X_{S_{t}} \right)^{-1} X_{S_{t}}^{\top}}_2  \frac{\sqrt{2} \norm{X_{V_{t}}}_2  \norm{ \left(X_{S_{t}}^{\top} X_{S_{t}} \right)^{-1} X_{S_{t}}^{\top} \epsilon_{S_{t}} }_2 + \sqrt{2}\|\epsilon_{V_{t}}\|_2}{1 - \sqrt{2}\eta}.
    \end{align*}

\subsection{Proof of Corollary~\ref{cor:torrent_gaussian}}
    We trivially have $|(S_t \cup \inliers)| \leq n$ and, since $|S_t|, |\inliers| \geq n-c_n$, we also have $|(\inliers \cap S_t)| \geq n - 2c_n$. Therefore,
    \begin{equation*}
        |V(S_t)| \leq 2 c_n.
    \end{equation*}
    Since $|\{S \subseteq \{1,\dots,n\} \text{ s.t. } |S| = a_n\}| = \binom{n}{a_n} = \binom{n}{c_n} \leq (en/c_n)^{c_n}$ and by Lemma~\ref{thm:concentration} and the union bound (see the proof of Theorem~\ref{lem:ransac_gaussian} for a similar argument), there exists $K>0$ such that with probability at least $1-\delta$ for all $t \in \mathbb{N}$ that
    \begin{equation*}
        \norm{\epsilon_{V(S_t)}}_2 \leq \sigma \sqrt{2 c_n} \left(1 + \sqrt{K\log(2en/c_n\delta)} \right).
    \end{equation*}
    With similar arguments and noting that $\norm{(X_{S_{t}}^{\top} X_{S_{t}} )^{-1} X_{S_{t}}^{\top}}_{\mathrm{F}} \leq \sqrt{d} \norm{(X_{S_{t}}^{\top} X_{S_{t}} )^{-1} X_{S_{t}}^{\top}}_2$ we have with probability at least $1-\delta$ for all $t \in \mathbb{N}$
    \begin{equation*}
        \norm{ \left(X_{S_{t}}^{\top} X_{S_{t}} \right)^{-1} X_{S_{t}}^{\top} \epsilon_{S_{t}} }_2 \leq \sigma \norm{\left(X_{S_{t}}^{\top} X_{S_{t}} \right)^{-1} X_{S_{t}}^{\top}}_2 \left(\sqrt{d} + \sqrt{2c_n K\log(2en/c_n\delta)}\right).
    \end{equation*}
    By Theorem~\ref{thm:tor} and the fact that $\norm{ \left(X_{S_{t}}^{\top} X_{S_{t}} \right)^{-1} X_{S_{t}}^{\top}}_2 = \left(\lambda_{\mathrm{min}}\left( X_{S_t}^T X_{S_t}\right)\right)^{-1/2}$ we have
    \begin{align*}
        \norm{\hat{\beta}^{n, a_n}_{\tor} - \beta} _2 &\leq \frac{\sigma}{\lambda_{a_n}(X)} \left( 1 + \frac{\sqrt{2}}{1 - \sqrt{2}\eta} \right)  \left(\sqrt{d} + \sqrt{2c_nK\log(2en/c_n\delta)}\right)
        +  \frac{ 2\sigma \sqrt{c_n} \left(1 + \sqrt{K\log(2en/c_n\delta)} \right)}{\lambda_{a_n}(X)(1 - \sqrt{2}\eta)}.
    \end{align*}
This yields the first result. If the rows of $X$ are i.i.d.~standard Gaussian random vectors, we have $\lambda_{a_n}(X) \in \Omega_{\mathbb{P}}(\sqrt{n})$
    (see \citet[Theorem 15]{bhatia2015robust}).
    This yields the second statement.

\subsection{Proof of Theorem~\ref{thm:end_improved_bfs} and Theorem~\ref{thm:end_improved_torrent}}
\begin{lemma}
\label{lem:decor_proof}
    Assume the setup of
    Section~\ref{sec:theoretical_g_an}, 
    let $c_n$ be a sequence satisfying, for all $n$, $|\outliers| \leq c_n$ and
    define the $\mathcal{U}_n$ 
    as in 
    Theorem~\ref{thm:end_improved_bfs}
    and
    let Assumption~\ref{ass:new} be satisfied. Then, the following three statements hold.
    \begin{enumerate}[label=(\roman*)]
        \item\label{proof:1} $|\mathcal{U}_n| \leq \left(\frac{en}{c_n}\right)^{c_n}$,
        \item\label{proof:2} for all $n \in \mathbb{N}$ the $n$ components of
        $\eta^n_{\phi}$
        are i.i.d.~centred Gaussian random variables with variance $\bar{\sigma}^2 = \sigma_{\eta}^2/n$,
        \item\label{proof:4} for all $\delta > 0$ there exists $c' > 0$ and $\bar{n} \in \mathbb{N}$ such that for all $n \geq \bar{n}$ it holds that
        \begin{equation*}
            \mathbb{P}\left[
            \min_{S \in \sampleset}\sqrt{\lambda_{\mathrm{min}}\left( (X^n_{\phi})_{S \setminus \outliers}^T (X^n_{\phi})_{S\setminus \outliers}\right)}
            \geq c' \right] \geq 1 - \delta.
        \end{equation*}
    \end{enumerate}
\end{lemma}

\begin{proof}
    Since $\{T_k^{\phi,n }(\eta)\}_{k \leq n}$ are linear combinations of independent centred Gaussians, they are jointly Gaussian with mean zero.
    By Assumption~\ref{ass:new}~\ref{ass:new:2} it holds that
    \begin{align*}
        \mathbb{E}\left[ (\eta^n_{\phi})_k (\eta^n_{\phi})_l \right] = \mathbb{E}\left[ T_k^{\phi,n }(\eta) T_l^{\phi,n }(\eta) \right] &= \frac{1}{n^2} \sum_{i,j=1}^n \phi_k(Ti/n) \phi_l(Tj/n) \mathbb{E}[\eta_{Ti/n} \eta_{Tj/n}]\\
        &=\frac{\sigma_{\eta}^2}{n} \mathbbm{1}\{k = l\}.
    \end{align*}
    This proves \ref{proof:2}.
    Furthermore, it holds by Sterling's inequality that 
    \begin{align*}
        |\mathcal{U}_n| &= \binom{n}{c_n} \leq \left(\frac{en}{c_n}\right)^{c_n}.
    \end{align*}
    This proves \ref{proof:1}.
    We now prove \ref{proof:4}. 
    Since for all $S \in \sampleset$ we have $|S| = n - c_n$ and $|\outliers| \leq c_n$
    it holds that for all $S \in \sampleset$ that 
    $|S \setminus \outliers| \geq n - 2c_n$. 
    Now, consider $\bar{n}$ from 
    Assumption~\ref{ass:new}~\ref{ass:new:4} and an arbitrary $n \geq \bar{n}$. 
    Let $S'_n$ be the set from Assumption~\ref{ass:new}~\ref{ass:new:4}. %
    Then $|S \setminus \outliers| \geq n - 2c_n$
    implies $|(S \setminus \outliers) \cap S'_n| \geq d$.
    Therefore,
we can choose an $S'' \subseteq (S \setminus \outliers) \cap S'_n$ with $|S''|=d$. By Assumption~\ref{ass:new}~\ref{ass:new:4} we then have
with probability at least $1-\delta$
    \begin{equation*}
        \min_{S \in \sampleset}\sqrt{\lambda_{\mathrm{min}}\left( (X^n_{\phi})_{S''}^T (X^n_{\phi})_{S''}\right)}
        \geq c'.
    \end{equation*}
    Since $\lambda_{\mathrm{min}}$ is superadditive for positive semi-definite matrices, with probability at least $1-\delta$,
        \begin{equation*}
            \min_{S \in \sampleset}\sqrt{\lambda_{\mathrm{min}}\left( (X^n_{\phi})_{S \setminus \outliers}^T (X^n_{\phi})_{S\setminus \outliers}\right)}
            \geq c'.
        \end{equation*}
\end{proof}

\begin{proof}[Proof of Theorem~\ref{thm:end_improved_bfs}]
    We want to apply Theorem~\ref{lem:ransac_gaussian}.
    Since for all $n \in \mathbb{N}$ and $j,i\leq T$ it holds that $\frac{1}{n}\sum_{l=1}^n \phi_l(Tj/n) \phi_l(Ti/n) = \mathbbm{1}{\{ i = j \}}$ 
    (by assumption)
    we have that
    \begin{align}
    \label{eq:345453}
    \begin{split}
        \norm{X^n_{\phi}}^2_2 &= \sum_{l=1}^n \frac{1}{n^2} \sum_{i,j=1}^n X_{iT/n} X_{jT/n} \phi_l(Tj/n) \phi_l(Ti/n)\\
        &= \frac{1}{n} \sum_{j=1}^n X_{jT/n}^2
    \end{split}
    \end{align}
    Since $\sup_{t \in [0,T]} \mathbb{E}[X_t^2] < \infty$ there exists $\bar{c}>0$ such that for all $n \in \mathbb{N}$
    \begin{equation*}
        \mathbb{E}\left[ \norm{X^n_{\phi}}^2_2 \right] = \mathbb{E}\left[ \frac{1}{n} \sum_{k=1}^n X_{kT/n}^2 \right] \leq \bar{c} < \infty.
    \end{equation*}
    By Markov's inequality we have for all $\delta > 0$ and for all $n \in \mathbb{N}$ that
    \begin{equation}
    \label{eq:kirh}
        \mathbb{P}\left[\norm{X^n_{\phi}}^2_2 \leq \bar{c} /\delta \right] \geq 1 - \delta,
    \end{equation}
    By \eqref{eq:kirh} and Lemma~\ref{lem:decor_proof}~\ref{proof:4} for all $\delta > 0$ there exist $v_1, v_2>0$ and $\bar{n}$ such that or all $n \geq \bar{n}$, with probability at least $1-\delta$, 
    for all $S \in \sampleset$, we have $||(X^n_{\phi})_{S \setminus \outliers}||_2 > v_1$ and $||(X^n_{\phi})_{S}||_2 < v_2$.
    By definition of $\sampleset$ we have that for all $S \in \sampleset$ and all $U_n \in \B$
    it holds that $|S|=|U_n| = n - c_n$ and therefore $\alpha_1(S, U_n, \delta) \leq 4 \sigma^2 \sqrt{n} \sqrt{2\log(2|\sampleset| /\delta)}$.
    Property~\ref{proof:1} and Property~\ref{proof:2}
    of Lemma~\ref{lem:decor_proof} then yield the claim by Theorem~\ref{lem:ransac_gaussian}.
\end{proof}

\begin{proof}[Proof of Theorem~\ref{thm:end_improved_torrent}]
    Equation \eqref{eq:ass_top} implies that for large $n$ with high probability we have 
    $$\eta \coloneqq \max_{S \subseteq \{1,\dots,n\} \text{ s.t. } |S| = n-c_n} \frac{\norm{(X^n_{\phi})_{V(S)}}_2}{\sqrt{\lambda_{\mathrm{min}}\left( (X^n_{\phi})_{S}^T (X^n_{\phi})_{S}\right)}} \leq 1/\sqrt{2}$$
    and by Lemma~\ref{lem:decor_proof}~\ref{proof:4} (using the fact that 
    $\lambda_{\mathrm{min}}$ is superadditive for positive semi-definite matrices)
    we know that for large $n$ we have with high probability that for all $S$ with $|S| = n-c_n$ the value of $\sqrt{\lambda_{\mathrm{min}}\left( (X^n_{\phi})_{S}^T (X^n_{\phi})_{S}\right)}$ is lower bounded. This yields the desired result by Corollary~\ref{cor:torrent_gaussian}.
\end{proof}

\subsection{Proof of Proposition~\ref{prop:ols_not_consistent}}
    It holds that
    \begin{equation*}
        \hat{\beta}^{\phi,n}_{\ols} - \beta = \frac{(X^n_{\phi})^{\top} \eta^n_{\phi}}{\norm{X^n_{\phi}}_2^2} + \frac{(X^n_{\phi})^{\top} o^n_{\phi}}{\norm{X^n_{\phi}}_2^2}.
    \end{equation*}
    By Lemma~\ref{lem:decor_proof}~\ref{proof:2} the first term vanishes in probability. The second term, however, is non-vanishing in many cases, for example, when $U$ is a band-limited process and $X_t = U_t + \hat{\epsilon}_t$ for all $t\in [0,T]$, where $\hat{\epsilon}$ is a band-limited process independent of $U$.

\section{Additional Results}

\subsection{Consistency of \ols in the i.i.d.~Setting}
\label{app:inconsistent}
Assume that we are in the i.i.d.~bounded adversarial outlier setting with Gaussian noise. More precisely, let $\{x_k\}_{k \in \mathbb{N}} \subset \mathbb{R}$ and $\{\epsilon_k\}_{k \in \mathbb{N}}  \subset \mathbb{R}$ be independent, i.i.d.~sets of random variables with strictly positive variances
and
$\epsilon_1$ having mean zero, and let
$\{o_k\}_{k \in \mathbb{N}} \subset \mathbb{R}$ be a set of random variables that might depend on $\{x_k\}_{k \in \mathbb{N}}$ but not on $\{\epsilon_k\}_{k \in \mathbb{N}}$. Assume that there exists $c_{\mathrm{b}} \in \mathbb{R}$ such that for all $k \in \mathbb{N}$ it holds that $|o_k| \leq c_{\mathrm{b}}$.
Assume there exists $\beta \in \mathbb{R}$ 
such that
for all $k \in \mathbb{N}$
\begin{equation*}
    y_k \coloneqq \beta x_k + \epsilon_k + o_k.
\end{equation*}
Fix $n \in \mathbb{N}$. Define $x^n \coloneqq (x_1,\dots,x_n)^{\top}$ and $o^n$, $y^n$ and $\epsilon^n$ analogously. If \ols without outliers is consistent, that is, 
$|(x^n)^{\top} \epsilon^n/\lVert x^n\rVert_2^2| \to 0$ 
in probability, and the fraction of outliers is vanishing, that is, $|\{k \leq n\ \ \vert \ o_k \neq 0\}|/n \to 0$, then the \ols estimator $\hat{\beta}^{n}_{\ols}$ is consistent even in the presence of outliers. 
More precisely, 
\begin{align*}
    \left|\hat{\beta}^{n}_{\ols} - \beta \right|&
    \leq 
    \left|\frac{(x^n)^{\top} \epsilon^n}{\norm{x^n}_2^2}\right| + \left| \frac{(x^n)^{\top} o^n}{\norm{x^n}_2^2} \right|\\
    &\leq \left|\frac{(x^n)^{\top} \epsilon^n}{\norm{x^n}_2^2}\right| + c_{\mathrm{b}}\frac{\sqrt{|\{k \leq n\ \ \vert \ o_k \neq 0\}|}/\sqrt{n}}{ \norm{x_n}_2 /\sqrt{n}}\\
    &\to 0,
\end{align*}
where the 
denominator of the right-hand term converges to 
a constant by the weak law of large numbers and the continuous mapping theorem.

\subsection{Example application of Theorem~\ref{lem:ransac_gaussian}}
\label{app:3453456}
If \eqref{eq:bfs_cons_cond_2} holds, all terms in the bound of the estimation error in Theorem~\ref{lem:ransac_gaussian} are asymptotically bounded by $\sqrt{\log(\sampleset)/\delta}/\sqrt{n}$, except the term that includes $\alpha_1$. If we multiply out the quadratic terms appearing in $\alpha_1$, we obtain
\begin{equation}
\label{eq:37467834}
    2\sigma^2 \log(K|\sampleset|/\delta) (|S|/|U_n|-1) + 2\sigma^2 \sqrt{2\log(K|\sampleset|/\delta)} \sqrt{|S|}+ 2
    \sigma^2 \sqrt{2\log(2|\sampleset|/\delta)} |S|/\sqrt{U_n}.
\end{equation}
When  minimizing over $U_n$ and maximizing over $S$ we can asymptotically bound \eqref{eq:37467834}  by $\log(2|\sampleset|/\delta) |S|/\sqrt{U_n}$.
Therefore, if \eqref{eq:bfs_cons_cond_2} holds,
then with probability at least $1-\delta$
\begin{equation*}
    \sqrt{n} \left| \hat{\beta}^n_{\bfs}(\sampleset) - \beta \right| 
    \in
    \mathcal{O}\left( \sqrt{|S|\log(|\sampleset|/\delta)/\sqrt{|U_n|}} \right).
\end{equation*}

\subsection{Lower Bound}
\label{app:lower_bound}
\begin{theorem}
\label{thm:lower_bound_simple}
    Let $\{c_n\}_{n \in \mathbb{N}}$ be a sequence of natural numbers
    and let $x = \{x_i\}_{i \in \mathbb{N}}$ be a sequence of real numbers. Let $\epsilon^1 = \{\epsilon^1_i \}_{i \in \mathbb{N}}$ have i.i.d.~uniformly on $[-1/2, 1/2]$ distributed elements. Define for all $i \leq n$
    and $\beta_1\in \mathbb{R}$ the random variables $$Y_i^1 \coloneqq \beta_1 x_i + \epsilon^1_i.$$
    Assume we observe 
    $\{ (x_i, \tilde{Y}_i)\}_{i \in \mathbb{N}}$,
    where $\tilde{Y}_1, \ldots, \tilde{Y}_n$
    are obtained by (adversarially) perturbing  %
    $c_n$
    of the $Y$-values.
    if there exist $c, o \in (0,1)$ and $\bar{n} \in \mathbb{N}$ such that for all $n \geq \bar{n}$
    \begin{equation}
    \label{eq:2342}
        p_n \coloneqq \sum_{i=1}^n \min\{1, c|x_i|\} < c_n (1-o),
    \end{equation}
    there does not exist a consistent estimator for $\beta_1$.
\end{theorem}
\begin{proof}
Let $\beta_2 \in \mathbb{R}$ and $n \in \mathbb{N}$. Let $\epsilon^2 = (\epsilon^2_1, \dots, \epsilon^2_n)^{\top} \subset [-1/2,1/2]^n$ have i.i.d.~uniformly distributed elements independent of $\epsilon_1$. Define for all $i \leq n$ the random variables $Y^2_i \coloneqq \beta_2 x_i + \epsilon^2_i$.
For $u \in \mathbb{R}$, let $A_1(u) \coloneqq \{v \in \mathbb{R} \ \vert \ |v-\beta_1 u| \leq 1 \}$ be the support of $Y \coloneqq \beta_1 u + \epsilon$, where $\epsilon$ is distributed uniformly in $[-1/2,1/2]$. Define $A_2(u) \coloneqq \{v \in \mathbb{R} \ \vert \ |v-\beta_2 u| \leq 1/2 \}$. Let $\mu$ be the Lebesgue measure and define the difference of support $Q_1(u)$ by $Q_1(u) \coloneqq A^1(u) \setminus A^2(u)$ and $Q_2(u)$ by $Q_2(u) \coloneqq A^2(u) \setminus A^1(u)$. We have $\mu(Q_1(u)) = \mu(Q_2(u)) = \min\{1, |\beta_1 - \beta_2|u\}$ %
Therefore, for all $i \leq n$ we find that $N^1_i \coloneqq \mathbbm{1}\{Y^1_i \in Q_1(x_i)\}$ and $N^2_i \coloneqq \mathbbm{1}\{Y^2_i \in Q_2(x_i)\}$ are independently Bernoulli distributed with parameter $\min\{1, |\beta_1 - \beta_2||x_i|\}$. Let $N^1$ 
be the number of datapoints for which 
$Y^1_i \in Q_1(x_i)$,
\begin{equation*}
    N^1 \coloneqq \sum_{i=1}^n N^1_i = \sum_{i=1}^n \mathbbm{1}\{Y^1_i \in Q_1(x_i)\}.
\end{equation*}
Define $N^2$ analogously. Let $K_n$ 
be the event that $N^1 \leq c_n$ and $N^2(x) \leq c_n$. Assume $K_n$ holds. Let $H^1_n, \dots, H^1_1$ and $H^2_1, \dots, H^2_n$ be i.i.d.~Bernoulli distributed with parameter $0.5$.
For all $B^1_1, \dots, B^1_n$ and $B^2_1, \dots, B^2_n$ be jointly independently distributed such that for all $i \leq n$ we have that $B^1_i$ is uniformly distributed on $Q_2(x_i)$ and $B^2_i$ is uniformly distributed on $Q_1(x_i)$.
Define for all $i \leq n$
\begin{equation*}
    Y'^1_i \coloneqq N^1_i \left( H^1_i Y^1_i + (1-H^1_i) B^1_i \right) + (1- N^1_i) Y^1_i.
\end{equation*}
Put simply, if $Y^1_i \in Q_1(x_i)$ with probability 50\%, we move the data point to $Q^1(x_i)$. We define $Y'^2_1, \dots, Y'^2_n$ analogously.
By construction we have that $Y'^1_1, \dots, Y'^1_n \mid K_n$ and $Y'^2_1, \dots, Y'^2_n \mid K_n$ have the same distribution. Therefore, under $K_n$, no estimator can differentiate between $\beta_1$ and $\beta_2$.

It remains to bound $1 - \mathbb{P}[K_n]$. We have
\begin{equation*}
    1 - \mathbb{P}[K_n] \leq  \mathbb{P}\left[ N^1 > c_n \right] + \mathbb{P}\left[ N^2 > c_n \right] = 2 \mathbb{P}\left[ N^1 > c_n \right].
\end{equation*}
We find that $N^1$ is Poisson binomial distributed with parameters $\min\{1, |\beta_1 - \beta_2||x_{1}|\}, \dots, \min\{1, |\beta_1 - \beta_2||x_n|\}$. Define $p_n \coloneqq \sum_{i=1}^n \min\{1, |\beta_1 - \beta_2||x_i|\}$.
By Chernoff's bound we have for $c_n \geq p_n$ 
that
\begin{equation*}
    \mathbb{P}\left[ N^1 > c_n \right] \leq \exp\left( c_n - p_n - c_n\log(c_n/p_n) \right).
\end{equation*}
Define $q_n \coloneqq p_n / c_n$, then
\begin{align*}
    \exp\left( c_n - p_n - c_n\log(c_n/p_n) \right) &=
    \exp\left( c_n \left(1- q_n - \log(1/q_n) \right)\right).
\end{align*}
If we have data such that there exists $\bar{n}$ and $o > 0$ such that for all $n \geq \bar{n}$ we have 
$q_n = p_n/ c_n < 1 - o$,
then $\mathbb{P}\left[ N^1 > c_n \right] \to 0$ and therefore $\mathbb{P}\left[ K_n\right] \to 1$ and no consistent estimator exists.
\end{proof}

\begin{remark}
    Assume that $X_1, X_2, \dots$ are i.i.d.~distributed such that for all $i \in \mathbb{N}$ we have $\mathbb{P}[X_i = 1] = \mathbb{P}[X_i = -1] = 1/2$ and let $c_n = \lfloor \eta n \rfloor$ for $\eta \in (0,1]$ (that is, $\eta$ is the fraction of outliers). Then $p_n \leq n |\beta_1 - \beta_2|$ and for $|\beta_1 - \beta_2|$ small enough there exists $o>0$ such that for all $n$,
    $p_n/\lfloor \eta n \rfloor < 1 - o$. Therefore, by Theorem~\ref{thm:lower_bound_simple}(in general) there does not exist a consistent estimator if we allow for
    a constant fraction of outliers.
\end{remark}

\section{Additional Information on the Experiments}

\subsection{Experimental Details}
\label{app:detail_experiments}
For the synthetic experiments, we generate data from the model specified in Setting~\ref{set:robust_improved} with $\beta=3$ and $T=1$.
We sample $\epsilon_X$ and $U$ either from two Ornstein-Uhlenbeck processes with parameters $(1, -0.8)$ and $(1, -0.5)$, respectively,
or from two band-limited processes
with coefficients
drawn i.i.d.~standard Gaussian with $S = \{1, \dots, 50\}$.
We choose the confounded components $\outliersnn$ uniformly at random with probability $0.25$.
For the othonormal basis $\phi$ we consider the cosine basis (see Definition~\ref{def:cosine_basis}) and the Haar basis (see Definition~\ref{def:haar_basis}).
We choose the threshold parameter $a$ for \torrent and \bfs as $0.7$.
We consider noise with variances $\sigma_{\eta}^2 \in \{0,1,4\}$.
We evaluate the mean absolute prediction error on 1000 sample sets, that is, $\mathrm{MAE} = \sum_{i=1}^{1000} |\beta - \hat{\beta}^{1000}_i|/{1000}$.
In all experiments, except for the one presented in Table~\ref{tab:bfs}, we use 
\algon-\tor.

\subsection{Additional Experiments}
\label{app:additional_exp}

In this section, we present additional experiments on synthetic data. Figure~\ref{fig:exp_n_False_haar} suggests that \algon-\tor is consistent when the Haar basis is used instead of the cosine basis. 
Figure~\ref{fig:exp_n_True_cosine_ablation}
considers 
a setting
with model misspecification. In Figure~\ref{fig:exp_n_True_cosine_ablation} (left) we plot the performance of \algon-\tor as the fraction of confounded datapoints increases. \algon-\tor seems to be consistent even when 50\% of the datapoints are confounded. As expected, when more than 50\% of the datapoints are confounded, \algon-\tor is not consistent.
In Figure~\ref{fig:exp_n_True_cosine_ablation} (right) we add standard Gaussian noise to $U$, which makes the model misspecified since Assumption~\ref{ass:main} is no longer satisfied for a `small' $G$.
In this sense, the confounding is dense (with $c_n = n$). 
However, since we are in the frequency domain, the variance of the noise converges  to $0$ (with increasing $n$), 
which we suspect to be the reason that \algon-\tor remains to appear consistent.
Figure~\ref{fig:exp_n_True_2} considers multivariate processes and suggests that \algon-\tor is consistent when $X$ is two-dimensional.
Lastly, Table~\ref{tab:convergence} shows that \algon-\tor converges in under 15 interations for sample sizes up to 1000 in the case where $X$ and $U$ are band-limited processes and the noise variance $\sigma_{\eta}$ is set to 1. 
\begin{figure}[ht]
     \centering
         \centering
         \includegraphics[width=0.45\textwidth]{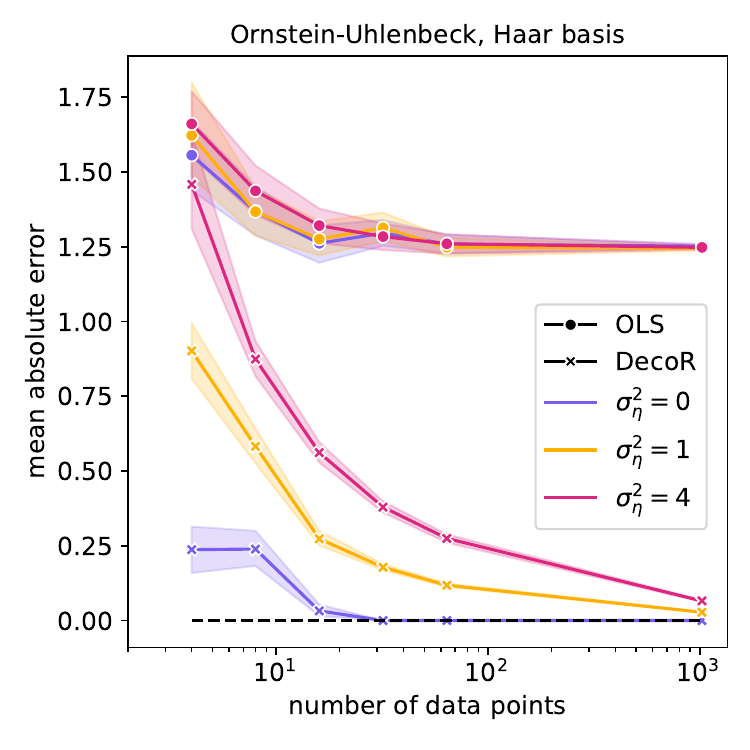}
         \caption{Synthetic experiment where $\epsilon_x$ and $U$ are generated by two independent Ornstein-Uhlenbeck processes and we choose $\phi$ to be the Haar basis. 
         }
         \label{fig:exp_n_False_haar}
\end{figure}

\begin{figure}[ht]
     \centering
     \begin{subfigure}[t]{0.45\textwidth}
         \centering
         \includegraphics[width=\textwidth]{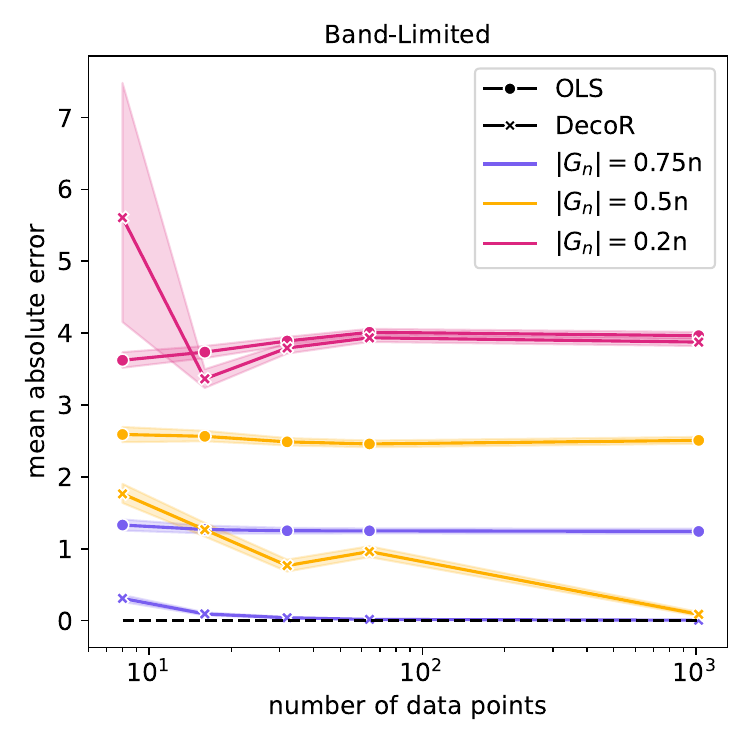}
     \end{subfigure}
     \hfill
     \begin{subfigure}[t]{0.45\textwidth}
         \centering
         \includegraphics[width=\textwidth]{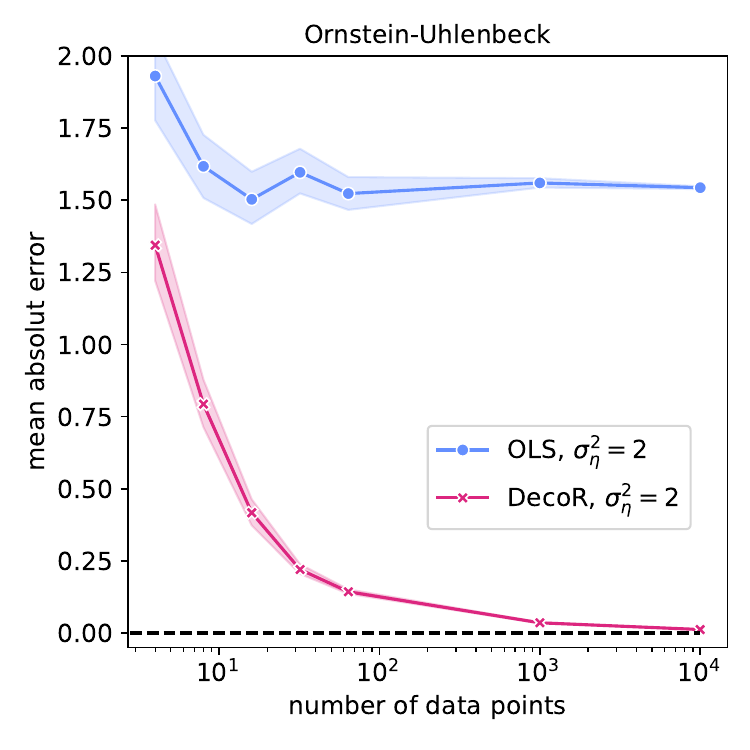}
     \end{subfigure}
    \caption{
         Left: Synthetic experiment where $\epsilon_x$ and $U$ are generated by two independent band-limited processes.
         Right: Synthetic experiment where $\epsilon_x$ and $U$ are generated by two independent Ornstein-Uhlenbeck processes. For this experiment we add i.i.d.~standard Gaussian noise to $U$, which makes the confounder dense and, therefore, the model is misspecified (cf. Setting~\ref{set:robust_improved}).
         }
    \label{fig:exp_n_True_cosine_ablation}
\end{figure}

\begin{figure}[ht]
     \centering
     \begin{subfigure}[t]{0.45\textwidth}
         \centering
         \includegraphics[width=\textwidth]{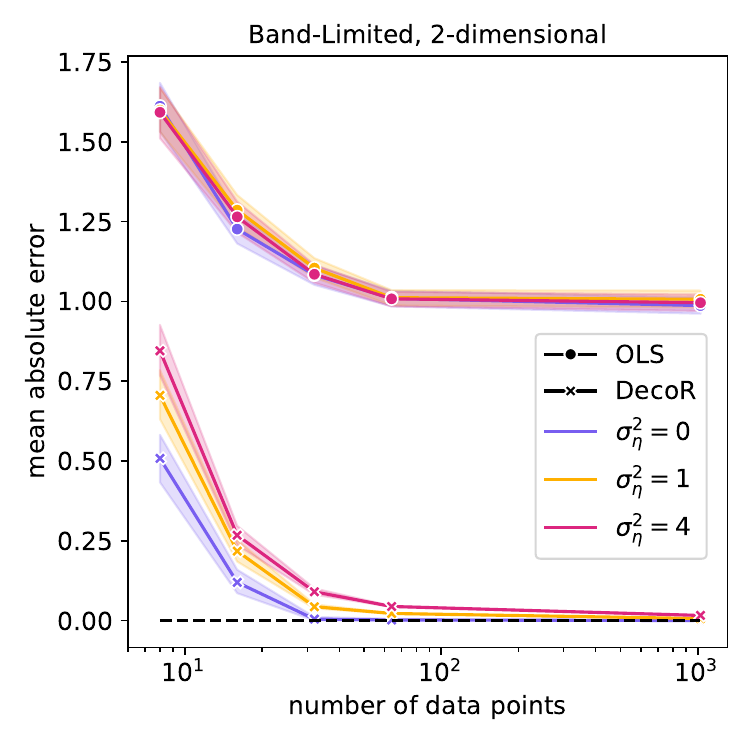}
     \end{subfigure}
     \hfill
     \begin{subfigure}[t]{0.45\textwidth}
         \centering
         \includegraphics[width=\textwidth]{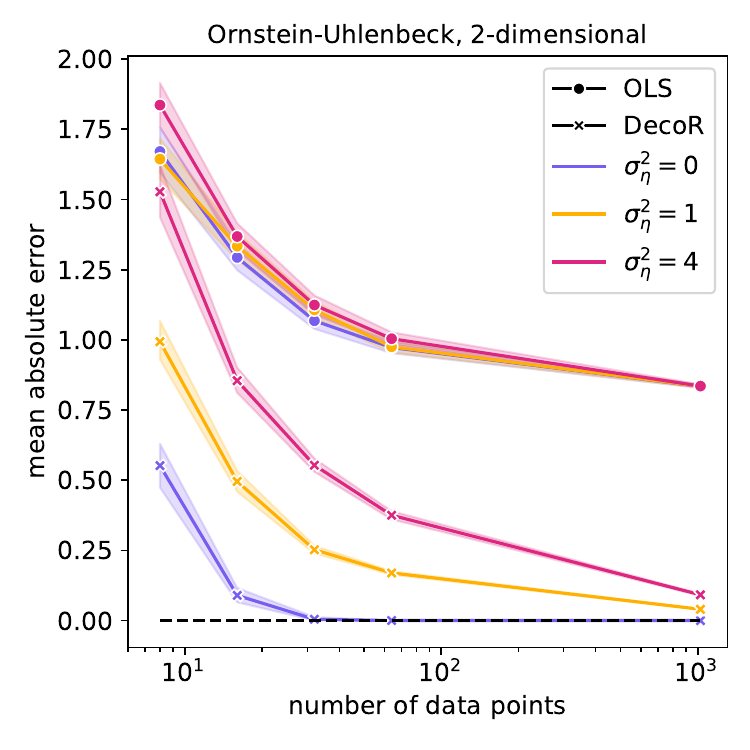}
     \end{subfigure}
    \caption{Left: Synthetic experiment where $\epsilon_x$ is generated by a 2-dimensional band-limited process and $U$ by a 1-dimensional band-limited process.
    Right: Synthetic experiment where $\epsilon_x$ is generated by a 2-dimensional Ornstein-Uhlenbeck process and $U$ by a 1-dimensional Ornstein-Uhlenbeck process.}
    \label{fig:exp_n_True_2}
\end{figure}

\begin{figure}[ht]
     \centering
        \begin{tabular}{ |c|c|c|c| } 
        \hline
        $n$ & mean & min & max \\
        \hline
        10 & 2.42 & 2 & 3 \\
        100 & 5.14 & 3 & 8 \\
        1000 & 8.26 & 4 & 13 \\
        \hline
        \end{tabular}
    \caption{Number of iterations until convergence for \algon-\tor. Even for sample size $n=1000$, \algon converges after $<15$ iterations. 
    }
    \label{tab:convergence}
\end{figure}

\end{document}